\documentclass{article}

\usepackage{microtype}
\usepackage{graphicx}
\usepackage{xcolor}
\usepackage{subcaption}
\usepackage{booktabs} 

\usepackage[hyphens]{url}
\usepackage[backref=page]{hyperref}

\usepackage[accepted]{icml2024}
\usepackage{enumitem}

\usepackage{amsmath}
\usepackage{amssymb}
\usepackage{mathtools}
\usepackage{amsthm}
\usepackage{bbm}
\usepackage{comment}
\usepackage{cancel}
\usepackage{xspace}
\usepackage[capitalize,noabbrev]{cleveref}
\usepackage{soul}
\definecolor{goodgreen}{RGB}{0,104,0} 

\theoremstyle{plain}
\newtheorem{theorem}{Theorem}[section]

\theoremstyle{definition}
\newtheorem{definition}[theorem]{Definition}

\theoremstyle{remark}

\newtheorem{innercustomthm}{Theorem}
\newenvironment{customthm}[1]
  {\renewcommand\theinnercustomthm{#1}\innercustomthm}
  {\endinnercustomthm}

\DeclareMathSymbol{\shortminus}{\mathbin}{AMSa}{"39}

\usepackage[textsize=tiny]{todonotes}

\definecolor{darkgreen}{rgb}{0.0, 0.5, 0.0}

\DeclareMathOperator*{\argmax}{arg\,max}
\DeclareMathOperator*{\argmin}{arg\,min}

\begin{document}

\twocolumn[

\icmltitle{Refining Minimax Regret for Unsupervised Environment Design}



\icmlsetsymbol{equal}{*}

\begin{icmlauthorlist}
\icmlauthor{Michael Beukman}{equal,oxeng}
\icmlauthor{Samuel Coward}{equal,oxeng}
\icmlauthor{Michael Matthews}{oxeng}\\
\icmlauthor{Mattie Fellows}{oxeng}
\icmlauthor{Minqi Jiang}{ucl}
\icmlauthor{Michael Dennis}{berkeley}
\icmlauthor{Jakob Foerster}{oxeng}

\end{icmlauthorlist}

\icmlaffiliation{oxeng}{University of Oxford}
\icmlaffiliation{ucl}{University College London}
\icmlaffiliation{berkeley}{UC Berkeley}

\icmlcorrespondingauthor{Michael Beukman}{mbeukman@robots.ox.ac.uk}
\icmlcorrespondingauthor{Samuel Coward}{scoward@robots.ox.ac.uk}

\icmlkeywords{Reinforcement Learning, Unsupervised Environment Design, Minimax Regret}

\vskip 0.3in
]

\newcommand{\perfregret}{\text{Regret}}
\newcommand{\mysetminus}{\Theta'}

\newcommand{\remidi}{ReMiDi\xspace}

\newcommand{\new}[1]{{\color{purple} #1}}
\newcommand{\newsm}[1]{{\textcolor{purple}{#1}}}

\newcommand{\hh}{\tau}
\DeclarePairedDelimiter\abs{\lvert}{\rvert}%
\theoremstyle{plain}
\newtheorem{observation}{Observation}[section]
\theoremstyle{definition}
\theoremstyle{remark}

\printAffiliationsAndNotice{\icmlEqualContribution} 

\begin{abstract}

    In unsupervised environment design, reinforcement learning agents are trained on environment configurations (levels) generated by an adversary that maximises some objective.
    Regret is a commonly used objective that theoretically results in a minimax regret (MMR) policy with desirable robustness guarantees; in particular, the agent's maximum regret is bounded.
    However, once the agent reaches this regret bound on all levels, the adversary will only sample levels where regret cannot be further reduced. Although there may be possible performance improvements to be made outside of these regret-maximising levels, learning stagnates.
    In this work, we introduce \textit{Bayesian level-perfect MMR} (BLP), a refinement of the minimax regret objective that overcomes this limitation.
    We formally show that solving for this objective results in a subset of MMR policies, and that BLP policies act consistently with a Perfect Bayesian policy over all levels.
    We further introduce an algorithm, \textit{ReMiDi}, that results in a BLP policy at convergence.
    We empirically demonstrate that training on levels from a minimax regret adversary causes learning to prematurely stagnate, but that ReMiDi continues learning.
    
\end{abstract}
\section{Introduction}\label{sec:intro}

Unsupervised environment design (UED) is an approach to automatically generate a training curriculum of environments for deep reinforcement learning (RL) agents~\citep{dennis2020Emergent,jiang2021Replayguided}. Regret-based UED trains an adversary to select environment configurations (referred to as \textit{levels}) that maximise the agent's \textit{regret}, i.e., the difference in performance between an optimal policy on that level and the agent. In other words, regret measures how much better a particular agent could perform on a particular level. Empirically, training on these regret-maximising levels has been shown to improve generalisation to out-of-distribution levels in challenging domains~\citep{dennis2020Emergent,jiang2021Replayguided, holder2022Evolving,samvelyan2023Maestro,team2023Humantimescale}. Furthermore, at equilibrium, UED methods theoretically result in a \textit{minimax regret policy}~\citep{dennis2020Emergent,jiang2021Replayguided}, meaning that the policy's worst-case regret is bounded.
\begin{figure}
    \centering
    \includegraphics[width=\linewidth]{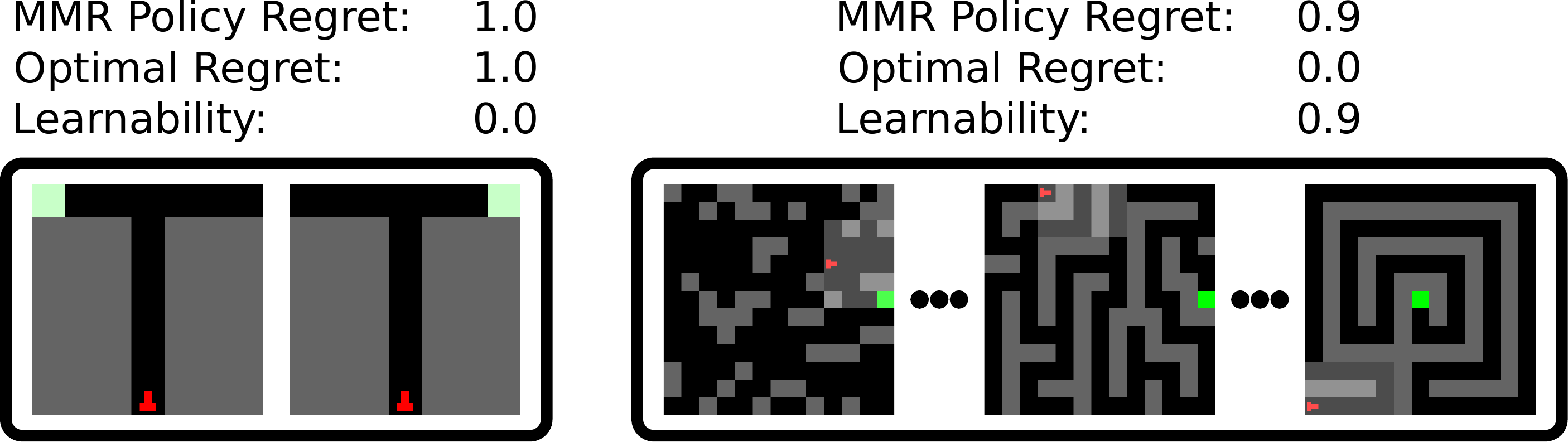}
    \caption{An illustration of the regret stagnation problem of minimax regret that our work addresses. In the T-mazes, the reward for reaching the goal is $1.0$ and $-1.0$ for failing. The reward for the mazes is $0.9$ for reaching the goal, and zero otherwise. Regret-based UED methods gravitate towards sampling high regret environments (T-mazes in this case with a regret of $1.0$), even if the agent cannot improve on these levels. This is despite the existence of non-high-regret levels (the mazes, with regret upper-bounded by $0.9$) on which the agent can still improve.}\label{fig:core_problem}
\end{figure}
Minimax regret (MMR) works well when the agent can simultaneously perform optimally on all levels: at convergence, the MMR policy would achieve zero regret for each level. 
However, this is not possible in all environments~\citep{sukhbaatar2017Intrinsic}.
As an example, consider the T-mazes in \cref{fig:core_problem}: these levels have different optimal behaviours but, due to partial observability, are indistinguishable to the agent. In this case, the agent cannot simultaneously perform optimally on both levels, and therefore suffers some \textit{irreducible regret}. 
Since MMR-based UED methods prioritise sampling the highest regret levels, these two levels will continually be sampled for training---even though they provide nothing more for the agent to learn.

This phenomenon is problematic if we have a subset of levels that (a) are distinguishable from the irreducible regret levels and (b) have lower, but reducible, regret; for instance, the set of all simultaneously solvable mazes in \cref{fig:core_problem} (which have regret of $0.9$ or lower).
A theoretically-sound UED method that implements MMR will converge to sampling each T-maze with 50\% probability, and fail to sample any of the solvable mazes. While this does technically satisfy the MMR objective, there is no guarantee we will achieve a policy that is effective at solving normal mazes, as the agent would only rarely see these, if at all. This shows a weakness of MMR---which we call \emph{regret stagnation}---because we know there exists a policy that obtains optimal regret on T-mazes and normal mazes.

We theoretically address this \textit{regret stagnation problem} by proposing a refinement of the MMR objective, which we call \emph{Bayesian level-perfect MMR} (BLP). Our objective aims to be minimax-regret over non-MMR levels, under the constraint that the policy must act according to the MMR policy in all trajectories that are consistent with MMR levels.
This process is iteratively repeated, each time over a smaller subset of levels. In this way, a BLP policy retains minimax regret guarantees, and iteratively improves its worst-case regret on the remainder of the levels.
We further show that any BLP policy acts consistently with a Perfect Bayesian policy on all levels.
Finally, we develop an algorithm, ReMiDi, that results in a BLP policy at convergence.

Our contributions are as follows:
\begin{enumerate}
    \item We theoretically introduce and characterise the regret stagnation problem in minimax regret UED.
    \item We propose BLP, a refinement of minimax regret for UED, that retains global minimax regret, and additionally obtains minimax regret-like guarantees under trajectories that do not occur in high-regret levels.
    \item We introduce a proof-of-concept algorithm, ReMiDi, that solves our new objective and returns a BLP policy.
    \item We empirically demonstrate that, in settings with high irreducible regret, ReMiDi significantly outperforms standard regret-based UED.
\end{enumerate}
By solving this problem, we empower the use of UED in larger and more open-ended settings, where irreducible regret is likely, as a BLP policy can be robust even outside these (potentially very rare) highest-regret levels.

\section{Background}
\subsection{UPOMDPs}
We consider an underspecified partially-observable Markov decision process~\citep[UPOMDP]{dennis2020Emergent} $\mathcal{M} = \langle A, O, \Theta, S, P_S, P_O, \mathcal{R}, \gamma \rangle$. 
Here $A$ is the action space, $O$ is the observation space, and $S$ is the state space. $\Theta$ is the space of underspecified parameters commonly referred to as \textit{levels}, $P_S: S \times A \times \Theta \to \Delta(S)$\footnote{$\Delta(X)$ is the set of all probability distributions over the set $X$.} is the level-conditional transition distribution. We denote the initial state distribution as $P_0:\Theta \to \Delta(S)$. In the partially observable setting, the agent does not directly observe the state, but an observation variable $o\in O$ that is correlated to the underlying state. $\mathcal{R}:S\times A\rightarrow \mathbb{R}$ is the scalar reward function, we denote instances of reward at time $t$ as $r_t=\mathcal{R}(s_t,a_t)$ and $\gamma$ is the discount factor. Each set of underspecified parameters $\theta \in \Theta$ indexes a particular POMDP called a \textit{level}. In our maze example in \cref{fig:core_problem}, the level determines the location of the goal and obstacles but dynamics such as navigating and the reward function remain shared across all levels.

At time $t$ the agent observes an action-observation history (or trajectory) $\hh_t=\langle o_0, a_0, \dots, o_{t-1}, a_{t-1}, o_t \rangle$ and chooses an action according to a trajectory-conditioned policy $a_t\sim \pi(\hh_t)$. We denote the set of all trajectory-conditioned policies as $\Pi\dot = \{\pi\vert \pi:\mathcal{T}\rightarrow\Delta(A)\}$ where $\mathcal{T}$ denotes the set of all possible trajectories. 
For any level $\theta$, the agent's goal is to maximise the expected discounted return (called \textit{utility}), which we denote as:
\begin{align*}
    U_\theta(\pi)\dot = \mathbb{E}_{\pi,\theta}\left[ \sum_{t=0}^T \gamma^t r_t \right],
\end{align*}
where $\mathbb{E}_{\pi,\theta}$ denotes the expected value on $\theta$ if the agent follows policy $\pi$~\citep{sutton2018reinforcement}.
We denote an optimal policy for level $\theta$ as $\pi^\star_\theta\in \argmax_{\pi'} U_\theta(\pi')$. 
Finally, similar to prior work~\citep{dennis2020Emergent}, for all proofs, we restrict our attention to finite and discrete UPOMDPs.

\subsection{Unsupervised Environment Design} \label{sec:minimax_regret_game}
Unsupervised Environment Design (UED) is posed as a two-player game, where an adversary $\Lambda$ selects levels $\theta$ for an agent $\pi$ to train on~\citep{dennis2020Emergent}. The adversary's goal is to choose levels that maximise some utility function, e.g., a constant utility for each level corresponds to domain randomisation~\citep{tobin2017Domain}. 
One commonly-used objective is to maximise the agent's regret~\citep{savage1951Theory,Bell82,Loomes82}. Formally, the \textit{regret} of policy $\pi$ with respect to an optimal policy $\pi_\theta^\star$ on a level $\theta$ is equal to how much better $\pi_\theta^\star$ performs than $\pi$ on $\theta$, $\perfregret_\theta(\pi) \dot = U_\theta(\pi_\theta^\star) - U_\theta(\pi)$.

If regret is used as the payoff, at Nash equilibrium of this two-player zero-sum game, the policy satisfies minimax regret~\citep{osborne2004Introduction,dennis2020Emergent}:\footnote{A Nash equilibrium always exists in finite games~\citep{nash1950Equilibrium}.}
\begin{equation}\label{eq:minimaxregret}
    \pi_\text{MMR} \in \Pi^\star_\text{MMR}\dot = \argmin_{\pi} \{\max_{\theta \in \Theta} \{ \perfregret_\theta(\pi) \} \}.
\end{equation}

Constraining policies to the set of MMR policies $\Pi^\star_\text{MMR}$ has several advantages: when deploying our policy, our regret can never be higher than the minimax regret bound, so the policy has a certain degree of robustness. Using minimax regret also results in an adaptive curriculum that increases in complexity over time, leading to the agent learning more efficiently~\citep{dennis2020Emergent,holder2022Evolving}.

Further, choosing levels based on maximising regret avoids sampling levels that are too easy (as the agent already performs well on these) or impossible (where the optimal policy also does poorly). This is in contrast to standard minimax, which tends to choose impossible levels that minimise the agent's performance~\citep{pinto2017Robust,dennis2020Emergent}.

A drawback of using minimax regret is the assumption of having access to the optimal policy $\pi^\star_\theta$ per level, which is generally unavailable. To circumvent this issue, most methods approximate regret in practice. We discuss several commonly-used heuristics in \cref{app:regret_metrics}. A more serious issue with using minimax regret in isolation is that there is no formal method to choose between policies in $\Pi^\star_\text{MMR}$. Typically it is chance and initialisation that determines the policy an algorithm converges to. While all minimax regret policies protect against the highest-regret outcomes, these events may be rare and there may be significant differences in the utility of policies in $\Pi^\star_\text{MMR}$ in more commonly encountered levels. We discuss this issue further in \cref{sec:limits_of_regret}. Finally, we refer the interested reader to \cref{app:additional_bg} for additional background on decision theory and UED.

\section{The Limits of Minimax Regret}\label{sec:limits_of_regret}
To elucidate the issues with using minimax regret in isolation, we analyse the set of minimax regret policies $\Pi^\star_\text{MMR}$ introduced in \cref{sec:minimax_regret_game}. For any $\pi_\text{MMR}\in \Pi^\star_\text{MMR}$ and $\vartheta \in \Theta$, it trivially holds that: 
\begin{align*}
\operatorname{Regret}_\vartheta(\pi_\text{MMR}) \le \min_{\pi'} \left\{\max_{\theta \in \Theta} \left\{{\operatorname{Regret}_\theta(\pi')} \right\}\right\}.
\end{align*}
However, it is unclear whether all policies in $\Pi^\star_\text{MMR}$ are equally desirable across all levels. 
In the worst case, the minimax-regret game will converge to an agent policy that only performs as well as this bound, \textit{even if further improvement is possible} on other (non-minimax regret) levels. In addition, the adversary's distribution will not change at such Nash equilibria, by definition. Thus, at equilibrium, the agent will not be presented with levels outside the support of $\Lambda$ and as such will not have the opportunity to improve further---despite the possible existence of other MMR policies with lower regret outside the support of $\Lambda$.

This observation, and the concrete example in \cref{fig:core_problem} demonstrate that minimax regret does not always correspond to \emph{learnability}: there could exist UPOMDPs with high regret on a subset of levels on which an agent is optimal (given the partial observability constraints), and low regret on levels in which it can still improve. \textit{Our key insight is that optimising solely for minimax regret can result in the agent's learning to stop prematurely, preventing further improvement across levels outside the support of MMR levels}. We summarise this \textit{regret stagnation problem} of minimax regret as follows:
\begin{enumerate}
    \item The minimax regret game is indifferent to which MMR policy is achieved on convergence; and
    \item Upon convergence to a policy in $\Pi^\star_\text{MMR}$, no improvements occur on levels outside the support of $\Lambda$.
\end{enumerate}

\section{Refining Minimax Regret}\label{sec:refinement}
Having described the regret stagnation problem of minimax regret, we now introduce a new objective to address it.
Concretely, we propose \textit{Bayesian level-perfect Minimax Regret} (BLP), our refinement of the MMR decision rule applied to UED. 
To describe this objective succinctly, we first introduce the notion of a \textit{realisable} trajectory, and the \textit{refined MMR game}. The refined game fixes a policy $\pi$ and set of levels $\Theta'$, and restricts the solution to act consistent with $\pi$ in all trajectories possible given $\pi$ and $\Theta'$, where behaviour for other trajectories can be chosen arbitrarily.

\begin{definition}\label{def:realisable}\textbf{(Realisable Trajectory):}
    For a set $\Theta'$ and policy $\pi$, $\mathcal{T}_{\pi}(\Theta')$ denotes the set of all trajectories that are possible by following $\pi$ on any $\theta \in \Theta'$. We call a trajectory $\hh$ realisable under $\pi$ and $\Theta'$ iff $\hh \in \mathcal{T}_{\pi}(\Theta')$.

\end{definition}

\begin{definition}\textbf{Refined Minimax Regret game:}
    Given a UPOMDP with level space $\Theta$, suppose we have some policy $\pi$ and some subset of levels  $\Theta' \subseteq \Theta$. 
    We introduce the \textbf{refined minimax regret game under $\pi$ and $\Theta'$}, a two-player zero-sum game between an agent and adversary where:
    \begin{itemize}
        \item the agent's strategy set is all policies of the form
        $$
            \pi'(a | \hh) = \begin{cases}
                \pi(a | \hh) \text{ if } \hh \in \mathcal{T}_{\pi}(\mysetminus) \\
                \bar{\pi}(a | \hh) \text{ otherwise},
            \end{cases}
        $$
        where $\bar{\pi}$ is an arbitrary policy; 
        \item the adversary's strategy set is $\overline{\Theta} \dot = \Theta \setminus \Theta'$;
        \item the adversary's payoff is $\perfregret_\theta(\pi')$.
    \end{itemize}
    In other words, $\pi'$ represents the set of policies that perform identically to $\pi$ in any trajectory possible under $\pi$ and $\mysetminus$.
    At Nash equilibrium of the refined game, the agent will converge to a policy that performs identically to $\pi$ under all levels in $\mysetminus$ (by definition), but otherwise will perform minimax regret optimally over $\overline{\Theta}$ with respect to these constraints. The adversary will converge onto a minimax regret equilibrium distribution with support only on levels in $\overline{\Theta}$.

    We note that the solution to the refined game is not guaranteed to obtain absolute best worst-case regret over $\overline{\Theta}$, since it is constrained to act according to $\pi$ when observing certain trajectories. However, we can guarantee that the Nash solution must obtain optimal worst-case regret compared to all other policies that also have this constraint.
\end{definition}

We now state and prove an important property of the equilibrium policy in the refined game: it must monotonically improve upon $\pi$ in the highest-regret $\theta \in \Theta \setminus \Theta'$, while maintaining $\pi$'s regret over $\mysetminus$.
\begin{theorem}
\label{thm:mwci}
Suppose we have a UPOMDP with level space $\Theta$. Let $\pi$ be some policy and $\Theta' \subseteq \Theta$ be some subset of levels. Let $(\pi', \Lambda')$ denote a policy and adversary at Nash equilibrium for the refined minimax regret game under $\pi$ and $\Theta'$. Then, (a) for all $\theta \in \mysetminus$,
$\perfregret_\theta(\pi') = \perfregret_\theta(\pi)$; and (b) we have,
$$
    \max_{\theta \in \Theta \setminus \Theta'} \left\{\perfregret_\theta(\pi')\right\} \le \max_{\theta \in \Theta \setminus \Theta'} \left\{\perfregret_\theta(\pi))\right\}.
$$

\end{theorem}
\begin{proof}
    (a) By definition, $\pi'$ must act according to $\pi$ on all trajectories possible under $\pi$ and $\mysetminus$. Therefore, the performance (and thus regret) of $\pi'$ must be identical to $\pi$ for all levels in $\mysetminus$.
    (b) In the refined MMR game, $\pi$ is trivially in the agent's strategy set. So regardless of what level the adversary plays, the agent can always play $\pi$. Thus the agent's best response can never be worse than $\pi$.
\end{proof}

\cref{thrm:minimax_refinement_theorem} next shows the benefits of iteratively refining an initial minimax regret policy.
\newcommand{\mmrrefinementtheorem}[1]{
Let $\langle \pi_1$, $\Lambda_1 \rangle$ be in Nash equilibrium of the minimax regret game.
Let $\langle \pi_i, \Lambda_i \rangle$ with $1 < i$ denote the Nash equilibrium solution to the refined minimax regret game under $\pi_{i-1}$ and $\Theta_i' = \bigcup_{j = 1}^{i-1} \operatorname{Supp}(\Lambda_j)$.#1
Then, for all $i \geq 1$, (a) $\pi_{i}$ is minimax regret and (b) we have
\begin{align}\label{eq:better}
    \max_{\theta \in \Theta \setminus \Theta_i'} \left\{\perfregret_\theta(\pi_{i})\right\} \le \max_{\theta \in \Theta \setminus \Theta_i'} \left\{\perfregret_\theta(\pi_{i-1}))\right\}.
\end{align}
Finally, (c) for $1 \leq j < i$ and $\theta \in \text{Supp}(\Lambda_j)$, $\perfregret_\theta(\pi_i) = \perfregret_\theta(\pi_j)$.
}
\begin{theorem}\label{thrm:minimax_refinement_theorem} \textbf{(Minimax Regret Refinement Theorem):}

\mmrrefinementtheorem{\footnote{$\operatorname{Supp}(\Lambda)$ denotes the support of $\Lambda$, i.e., the set of all environments that it samples with nonzero probability.}}
In other words, iteratively refining a minimax regret policy (a) retains minimax regret guarantees; (b) monotonically improves worst-case regret on the set of levels not already sampled by any adversary; and (c) retains regret of previous refinements on previous adversaries.
\end{theorem}
This result holds due to inductively applying \cref{thm:mwci}. At every iteration of the refined game, we must perform equally on all previously chosen levels, by definition. On the non-chosen levels, we must either maintain or improve the regret bound of the previous step, by \cref{thm:mwci}. For space reasons, the formal proof is in \cref{sec:catchall}.

We have now defined the refined game, and shown that iteratively solving it retains minimax regret guarantees \textit{and} monotonically improves worst-case regret in non-MMR levels. We next define the solution concept we propose to use instead of standard minimax regret, which we call \textit{Bayesian level-perfect Minimax Regret} (BLP MMR). Intuitively, a BLP policy solves all of the refined games until all levels have been sampled.
\begin{definition} \label{def:level_perfect}
\textbf{Bayesian level-perfect Minimax Regret:}
Let $\langle \pi_1, \Lambda_1 \rangle$ be in Nash equilibrium of the minimax regret game.
Let $\langle \pi_i, \Lambda_i \rangle$, $1 < i$ denote the solution to the refined game under $\pi_{i-1}$ and $\bigcup_{j = 1}^{i-1} \operatorname{Supp}(\Lambda_j)$.
Policy $\pi_j$ is a \textit{Bayesian level-perfect minimax regret policy} if $\bigcup_{k=1}^{j} \operatorname{Supp}(\Lambda_k) = \Theta$.\footnote{We refer the interested reader to \cref{app:temporal_inconsistency}, where we discuss simpler, but ultimately flawed alternatives to BLP.}
\end{definition}

Next, \cref{thrm:bayes_perf} shows that BLP is similar to a Bayes perfect refinement of minimax regret, in that it acts consistently with a Perfect Bayesian policy under a minimax regret prior over levels. Its proof can be found in \cref{sec:proof:bayes_perf}.
\newcommand{\bayesperfecttheorem}{
A Bayesian level-perfect minimax regret policy $\pi'$ acts consistently with a Perfect Bayesian policy on all realisable trajectories under $\pi'$ and $\Theta$. 
}
\begin{theorem}\label{thrm:bayes_perf}
    \bayesperfecttheorem
\end{theorem}

\section{\remidi}
Having introduced our refinement of minimax regret optimality, we now introduce our proof-of-concept method \textit{\textbf{Re}fining \textbf{Mi}nimax Regret \textbf{Di}stributions} (ReMiDi, see \cref{alg:refinement_alg}) that learns a \textit{Bayesian level-perfect minimax regret policy}. \cref{fig:mainfig} illustrates how our approach works.

\begin{algorithm}[h]
    \caption{\textbf{Re}fining \textbf{Mi}nimax Regret \textbf{Di}stributions}\label{alg:refinement_alg}
    \begin{algorithmic}[1]
        \REQUIRE Level space $\Theta$
        \STATE Initialise agent $\mathcal{S} = \{\}$ and adversary buffers $\mathcal{A} = \{\}$
        \STATE $i \gets 1$
        \WHILE {$\bigcup_{j=1}^{i} \operatorname{Supp}(\Lambda_j) \neq \Theta$}
            \STATE Initialise adversary $\Lambda_i$ and policy $\pi_i$.
            \WHILE{$(\pi_i, \Lambda_i)$ is not at Nash equilibrium} \label{alg:line:nash}
                \STATE Sample level $\theta \sim \Lambda_i$
                \STATE $\pi' \gets \texttt{Combine}(\mathcal{S} \cup \{ \pi_i \}, \mathcal{A})$
                \STATE Sample trajectory $\hh$ using $\pi'$ and $\theta$.

                \IF{$\hh \not \in \mathcal{T}_\pi( \cup_{j=1}^{i-1} \operatorname{Supp}(\Lambda_j))$}
                    \STATE Update $\pi_i$ and $\Lambda_i$ using $\theta$ and $\hh$
                \ENDIF
                
            \ENDWHILE
            \STATE $\mathcal{S} \gets \mathcal{S} \cup \{\pi_i\}$
            \STATE $\mathcal{A} \gets \mathcal{A} \cup \{\Lambda_i\}$
            \STATE $i \gets i + 1$
        \ENDWHILE
        \STATE \textbf{return} $\texttt{Combine}(\mathcal{S}, \mathcal{A})$
    \end{algorithmic}
\end{algorithm}
\begin{figure*}[h]
    \centering
    \includegraphics[width=1\linewidth]{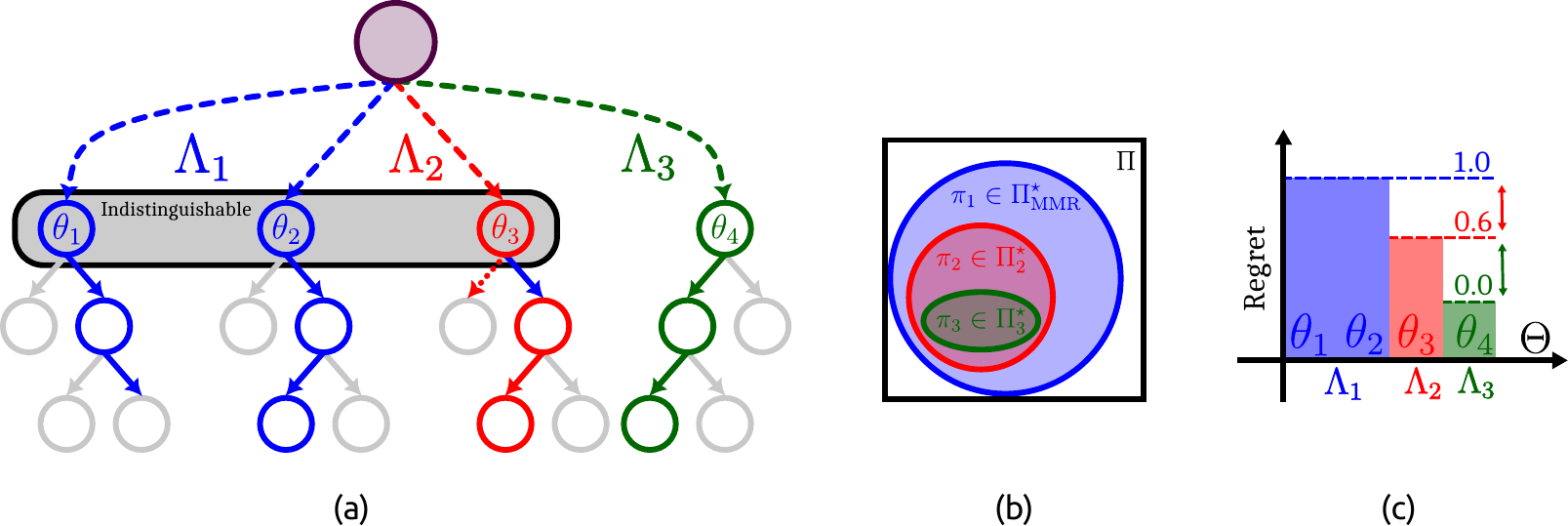}
    \caption{
        The BLP solution concept iteratively restricts the sets of policies by altering behaviour only in certain trajectories.
        \textbf{(a)} Each node corresponds to a level-trajectory pair, and the root node indicates that the adversary samples some level $\theta$.
        MMR results in adversary $\textcolor{blue}{\Lambda_1}$ and policy $\textcolor{blue}{\pi_1}$, reaching the nodes in \textcolor{blue}{blue}.
        MMR would terminate at step 1, but we instead refine our policy further. In step 2, we learn $\textcolor{red}{\Lambda_2}$ and $\textcolor{red}{\pi_2}$. 
        Following $\textcolor{blue}{\pi_1}$ in $\textcolor{red}{\theta_3}$ leads to a trajectory that never happens under $\textcolor{blue}{\Lambda_1}$ and $\textcolor{blue}{\pi_1}$.
        In step 3, we fill in behaviour for $\textcolor{goodgreen}{\theta_4}$, as these trajectories are never reached under any of the previous MMR adversaries. We terminate after all environments have been sampled by an adversary.
        \textbf{(b)} Iterative refinement reduces the set of policies we consider, improving upon the initial MMR policy $\textcolor{blue}{\pi_1}$.
        \textbf{(c)} If we have a minimax regret adversary $\textcolor{blue}{\Lambda_1}$, we are only guaranteed that the regret on all other levels must be at or below the regret of levels in the support of $\textcolor{blue}{\Lambda_1}$ (indicated by the dashed blue line). Refining our policy improves the bound on all levels except those sampled by $\textcolor{blue}{\Lambda_1}$. We iterate this process until all levels have been sampled, monotonically improving the regret bound on all non-previously-sampled levels.
        }\label{fig:mainfig}
\end{figure*}

In \cref{alg:refinement_alg}, $\texttt{Combine}(\mathcal{S}, \mathcal{A})$ returns $\pi_{n}'$, with $n = |\mathcal{S}|$, which is constructed as follows, with $\pi_1' = \pi_1$ being minimax regret:

{
$$
    \pi_n'(a | \hh) = \begin{cases}
        \pi'_{n-1}(a | \hh) &\text{if } \hh \in \mathcal{T}_{\pi'_{n-1}}( \bigcup_{i=1}^{n-1}\operatorname{Supp}(\Lambda_i)))\\
        \pi_{n}(a | \hh) &\text{otherwise.}\\
    \end{cases}
$$
}

\textbf{Outer Loop.} \cref{alg:refinement_alg} is a direct implementation of the successive refined MMR games. Thus, at convergence, it will return a BLP policy. The outer loop continues until all levels have been sampled by an adversary. In practice, this is both infeasible and excessive. Thus, one may choose to only compute a fixed number of outer iterations. 
Lines 9-11 ensure that the current adversary only contains levels with trajectories inconsistent with any previous adversary.

\textbf{Checking for convergence in the inner loop.} Line 5 requires convergence to Nash equilibrium, which is not guaranteed to occur in practice; even if equilibrium is reached, determining when this has happened is also non-trivial. One way to approximately determine convergence is to measure the expected regret over the adversary, and if this has plateaued, we can terminate the optimisation and continue to the next level. We use another, simpler, technique of having a fixed number of timesteps that each adversary is trained for, and assuming that (approximate) convergence is reached after training for this number of timesteps. 

\newpage
\textbf{Choice of adversary.}
\cref{alg:refinement_alg} is agnostic to the choice of the adversary; it could be an antagonist, level generator pair as used in PAIRED~\citep{dennis2020Emergent}, or a curated level buffer (e.g., PLR~\citep{jiang2020Prioritized}). 
In the case of PAIRED, only one antagonist is required for all generators.

\section{Experiments}
\subsection{Experimental Setup}
In this section, we empirically demonstrate that the problems identified in \cref{sec:limits_of_regret} do occur, and that \remidi alleviates these issues. 
First, in \cref{sec:exp_toy_settings}, we illustrate some of the failure cases of ideal UED in a simple tabular setting. 
Next, in \cref{sec:exp_maze}, we experiment in the canonical Minigrid domain.
In \cref{sec:lever}, we consider a different setting where regret-based UED results in a policy that performs poorly over a large subset of levels. Finally, we evaluate on a robotics task in \cref{sec:brax}.
To compare against an ideal version of UED, we use perfect regret as our score function. All plots show the mean and standard deviation over 10 seeds, whereas the tabular experiments use 50.\footnote{We publicly release our code at \url{https://github.com/Michael-Beukman/ReMiDi}.}

For the latter experiments, we compare against Robust PLR~\citep[$\text{PLR}^\perp$]{jiang2021Replayguided}, which is based on curating randomly-generated levels into a buffer of high-regret levels. At every step, the agent is either trained on a sample of levels from the buffer, or evaluated on a set of randomly-generated levels. These randomly generated levels replace existing levels in the buffer that have lower regret scores. In robust PLR, the agent does not train on randomly generated levels. Our second baseline is a minimax return adversary, which represents a large class of robust RL and worst-case methods~\citep{iyengar2005Robust,pinto2017Robust}. We implement this baseline as PLR$^\perp$ with a different score function: the negative of the agent's return. The buffer therefore contains levels that the agent obtains a low return on, and implicitly optimises agent performance on the worst-case levels.
Our ReMiDi implementation maintains multiple buffers, and we perform standard $\text{PLR}^\perp$ on the first buffer for a certain number of iterations. 
We then perform $\text{PLR}^\perp$ again, but reject levels that have complete trajectory overlap with levels in a previous buffer. 
Instead of explicitly maintaining multiple policies, we have a single policy that we update only on the parts of trajectories that are distinguishable from levels in previous buffers, approximately maintaining performance on previous adversaries.
\cref{alg:refinement_alg} assumes knowledge of whether a trajectory is possible given a policy and a set of levels, and we can compute this exactly in each environment.
We use JaxUED~\citep{coward2024JaxUED} for our experiments, and \cref{app:experimental_details} contains more implementation details.

\subsection{Exact Settings}\label{sec:exp_toy_settings}

We consider a one-step tabular game, where we have a set of $N$ levels $\theta_1, \dots, \theta_N$. Each level $i$ corresponds to a particular initial observation $\hh_i$, such that the same observation may be shared by two different environments. Each level also has an associated reward for each action $a_j, 1 \leq j \leq M$.

We model the adversary as a $N$-arm bandit, implemented using tabular Actor-Critic. Each action corresponds to a different level $\theta_i$. 
For \remidi, we have a sequence of adversaries, each selecting levels where the observations are disjoint with any previous adversary. 
In both cases, the agent is also a tabular Actor-Critic policy, with different action choices for each observation (equivalent to a trajectory) $\hh$.

\subsubsection{When Minimax Regret is Sufficient}
In \cref{fig:exp:toy:ued_works:standard_ued}, we first consider a case where MMR has none of the problems discussed in \cref{sec:limits_of_regret}. Here, each $\theta$ has a unique initial observation $\hh$, thus the level can be deduced solely from this observation. Minimax regret succeeds and converges to the globally optimal policy. Convergence occurs because a single policy can be simultaneously optimal over the set $\Theta$, as for every observation, there is one optimal action. The MMR policy is therefore also unique.

\subsubsection{When Minimax Regret Fails}
We next examine a UPOMDP where a single policy can no longer be simultaneously optimal over all levels. The setup is the same as the previous experiment, except that $\hh_2 = \hh_1$, $\hh_4 = \hh_3$, etc., meaning that there is some irreducible regret.
\cref{fig:exp:toy:ued_fails:both:regret} shows that regret-based UED rapidly obtains minimax regret, but fails to obtain optimal regret on the non-regret-maximising levels. By contrast, \remidi obtains optimal regret on all levels. 
It does this by first obtaining global minimax regret, at which point it restricts its search over levels to those that are distinguishable from minimax regret levels. Since the agent's policy is not updated on these prior states, it does not lose MMR guarantees.
\begin{figure}[H]
    \centering
    \includegraphics[width=0.4\textwidth]{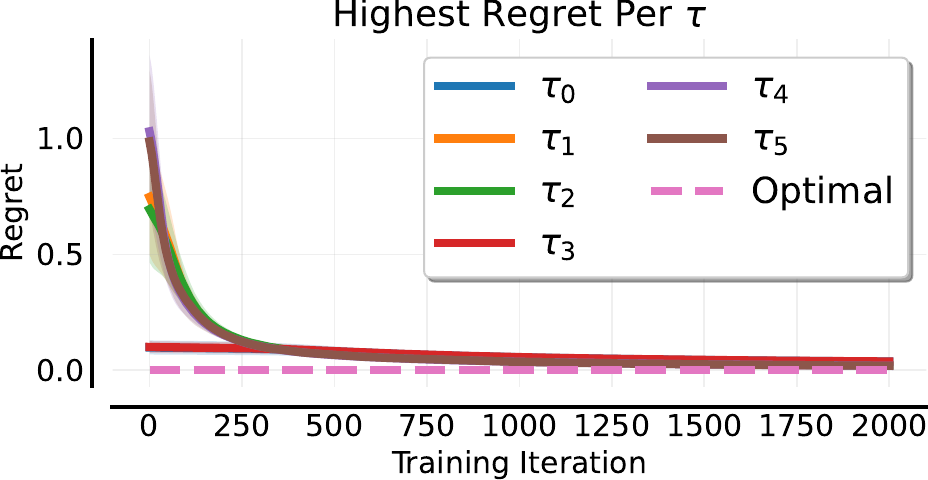}
    \caption{Plotting the regret of MMR UED throughout training for each of the $6$ trajectories. Here the optimal regret is $0$ for each trajectory, and minimax regret achieves this.}
    \label{fig:exp:toy:ued_works:standard_ued}
\end{figure}

We analyse this further in \cref{fig:exp:toy:ued_fails:both:prob} by plotting the probability of each level being sampled over time. Regret-based UED rapidly converges to sampling only the highest-regret levels ($\theta_4$ and $\theta_5$), and shifts the probability of sampling the other levels to zero. By contrast, our multi-step process first samples these high-regret levels exclusively. Thereafter, these are removed from the adversary's options and it places support on all other levels.
This shows that, while we could improve the performance of regret-based UED by adding stronger entropy regularisation~\citep{mediratta2023Stabilizing}, or making the adversary learn slower, the core limitation remains: when regret does not correspond to learnability, minimax regret UED will sample inefficiently.

\begin{figure*}[h]
    \begin{minipage}{0.48\linewidth}
        \centering
        \includegraphics[width=1\linewidth]{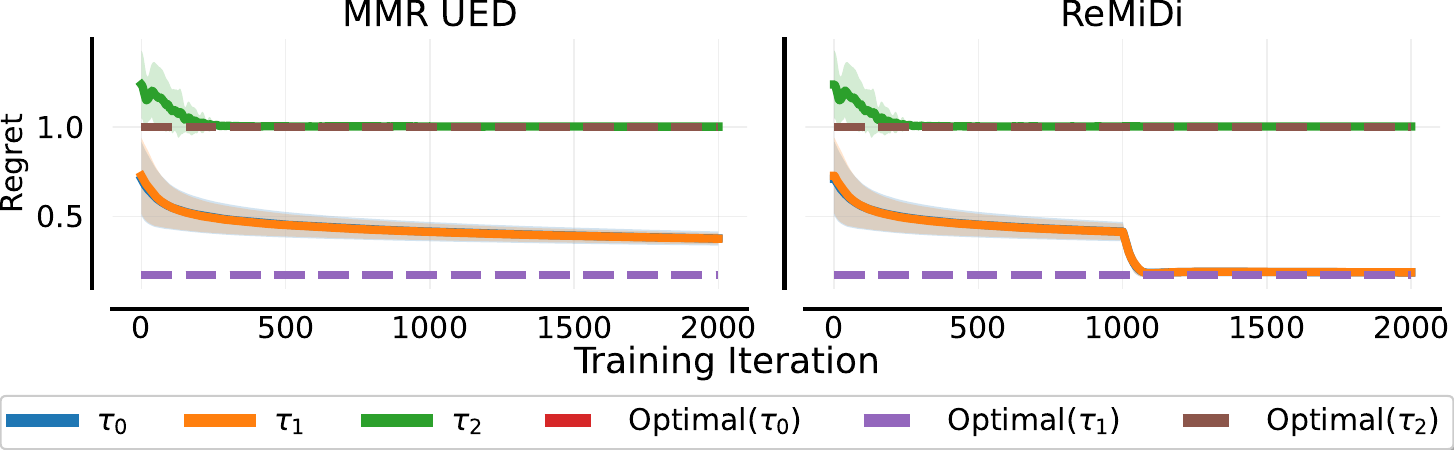}
        \caption{The greatest regret on each trajectory $\hh$ for (left) Standard Minimax Regret UED; and (right) ReMiDi. ReMiDi obtains optimal regret on all levels, whereas MMR does not.}
        \label{fig:exp:toy:ued_fails:both:regret}
    \end{minipage}\hfill%
    \begin{minipage}{0.48\linewidth}
        \centering
        \includegraphics[width=1\linewidth]{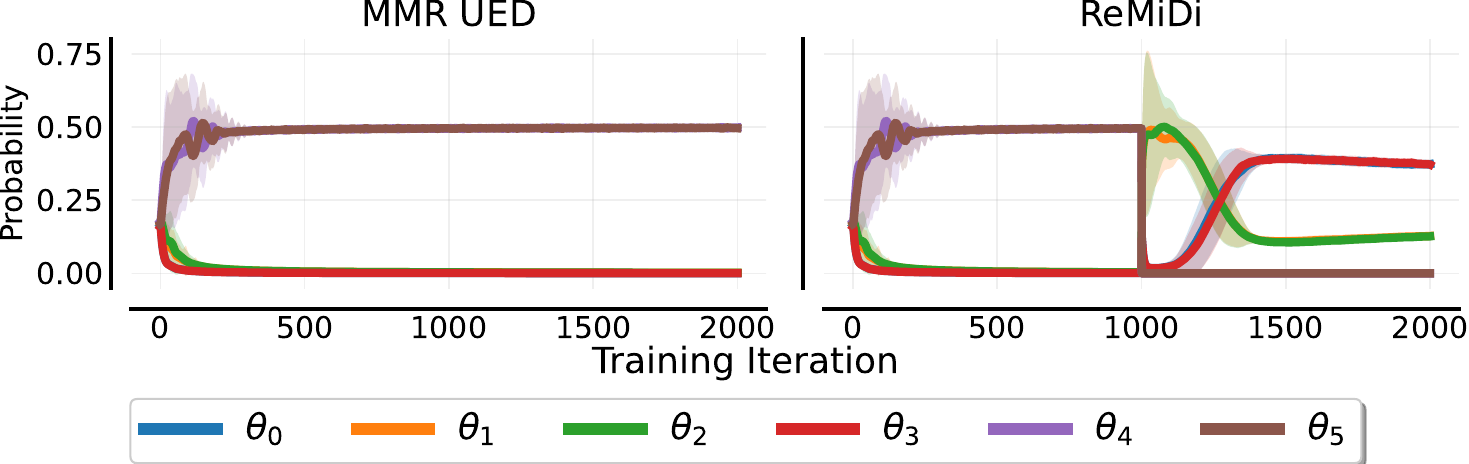}
        \caption{The Adversary's probability of sampling each environment $\theta$ for (left) Standard Minimax Regret UED; and (b) ReMiDi. MMR UED exclusively samples the irreducible-regret levels.}
        \label{fig:exp:toy:ued_fails:both:prob}
    \end{minipage}
\end{figure*}

\subsection{Maze}\label{sec:exp_maze}
We next consider Minigrid, a common benchmark in UED~\citep{dennis2020Emergent,jiang2021Replayguided,holder2022Evolving}. We discuss two distinct experimental settings.
The first is an implementation of the T-maze example in \cref{sec:intro}. Here the adversary can sample T-mazes or normal mazes. The reward of T-mazes is $+1$ or $-1$ depending on whether the agent reaches the goal or not, and the standard maze reward is the same as is used in prior work~\citep{dennis2020Emergent,jiang2021Replayguided}.
The second experiment is where the adversary has the choice of blindfolding the agent; in other words, it can zero out the agent's observation.
In both cases, we evaluate on a standard set of held-out mazes.

\cref{fig:exp:maze:tmaze} shows that $\text{PLR}^\perp$ with perfect regret as its score function results in poor performance on actual mazes. The reason for this is that it trains almost exclusively on T-mazes, and not on actual mazes (see \cref{fig:exp:maze:tmaze:prob}). ReMiDi, by contrast, samples T-mazes initially, and thereafter does not, as they have identical observations with previous MMR levels. This results in improved performance on actual mazes.

The blindfold experiment (\cref{fig:exp:maze:blind}) shows a similar result. ReMiDi performs better than $\text{PLR}^\perp$; again, this is because $\text{PLR}^\perp$ trains the agent almost exclusively on blindfold levels (\cref{fig:exp:maze:blind:prob}), as these have high irreducible regret. Interestingly, $\text{PLR}^\perp$, despite training almost exclusively on blind levels, still manages to solve many test mazes, a phenomenon investigated by \citet{wijmans2023Emergence}.

\begin{figure}[h]
    \begin{subfigure}{0.48\linewidth}
        \centering
        \includegraphics[width=1\linewidth]{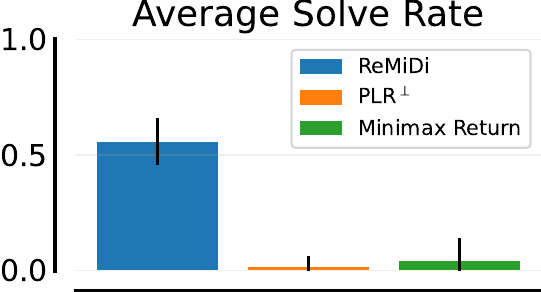}
        \caption{}
        \label{fig:exp:maze:tmaze}
    \end{subfigure}\hfill%
    \begin{subfigure}{0.48\linewidth}
        \centering
        \includegraphics[width=1\linewidth]{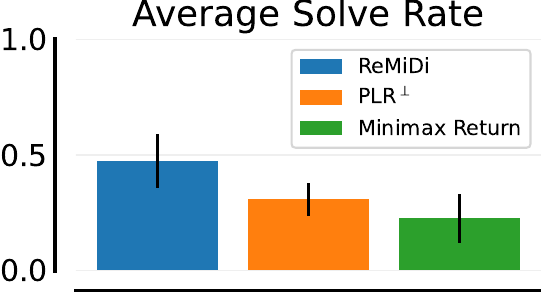}
        \caption{}
        \label{fig:exp:maze:blind}
    \end{subfigure}
    \caption{Solve rate on a set of held-out test mazes for the (a) T-maze and (b) blindfold experiments. ReMiDi is more effective at solving mazes than either $\text{PLR}^\perp$ or Minimax return.}
    \label{fig:exp:maze:tmaze_blind}
\end{figure}

\begin{figure}[h]
    \begin{subfigure}{0.48\linewidth}
        \centering
        \includegraphics[width=1\linewidth]{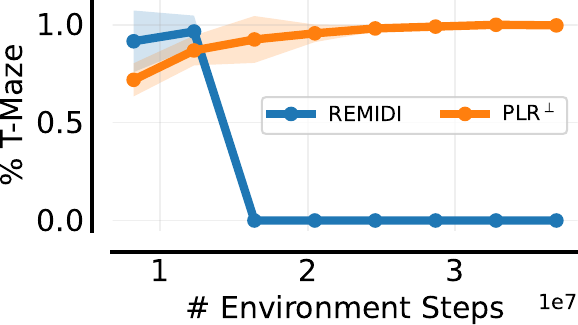}
        \caption{}
        \label{fig:exp:maze:tmaze:prob}
    \end{subfigure}\hfill%
    \begin{subfigure}{0.48\linewidth}
        \centering
        \includegraphics[width=1\linewidth]{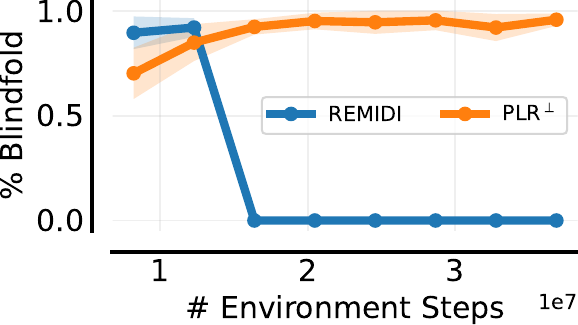}
        \caption{}
        \label{fig:exp:maze:blind:prob}
    \end{subfigure}
    \caption{Fraction of replay levels that are (a) T-mazes or (b) blindfolded. $\text{PLR}^\perp$ converges to almost never sampling normal mazes.}
    \label{fig:exp:maze:tmaze_blind:prob}
\end{figure}

\subsection{Lever Game}\label{sec:lever}
\begin{figure}[h]
    \begin{subfigure}{0.48\linewidth}
        \centering
        \includegraphics[width=1\linewidth]{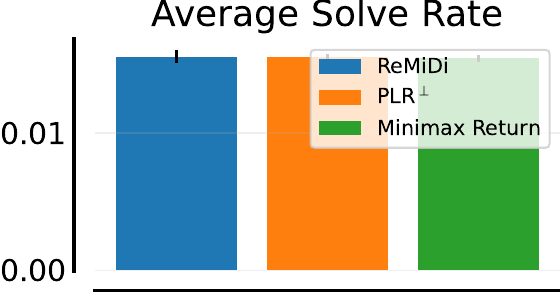}
        \caption{}
        \label{fig:exp:maze:lever:invis}
    \end{subfigure}\hfill%
    \begin{subfigure}{0.48\linewidth}
        \centering
        \includegraphics[width=1\linewidth]{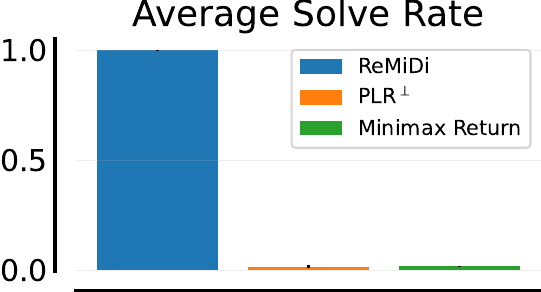}
        \caption{}
        \label{fig:exp:maze:lever:vis}
    \end{subfigure}
    \caption{Solve rate in the lever game for two subsets of evaluation levels: (a) invisible; and (b) visible answers. ReMiDi can solve both types, whereas  $\text{PLR}^\perp$ cannot, as it trained almost exclusively on invisible-answer levels.}
    \label{fig:exp:maze:lever:both}
\end{figure}
In this environment, inspired by \citet{hu2020Otherplay}, there are $64$ levers to pull, one of which is correct (reward of $+1$), and pulling a wrong lever results in a reward of $-1$. The adversary can make the correct lever known or unknown to the agent. In the latter case, the reward is multiplied by 10 (to simulate a harder problem having a higher reward).
Our analysis in \cref{sec:limits_of_regret} suggests that regret-based UED should solely sample levels where the correct answer is unknown, and the best option for the agent is to guess randomly (as this induces irreducible regret). Training solely on these levels, however, would cause the agent to perform poorly when it observes the correct answer.
Indeed, \cref{fig:exp:maze:lever:both} shows that on levels where the correct lever is not given---and the best policy guesses randomly---$\text{PLR}^\perp$ performs the same as ReMiDi, with a solve rate of around $\frac{1}{64} \approx 0.015$ (since there are 64 levers). 
On levels where the correct answer is given, however, ReMiDi performs perfectly (as it is possible to obtain 100\% accuracy), but $\text{PLR}^\perp$ fails as it almost never trained on these types of levels. 
Importantly, this result shows that both $\text{PLR}^\perp$ and ReMiDi satisfy minimax regret, but $\text{PLR}^\perp$ results in a policy that is effectively random, whereas ReMiDi learns a much more useful policy.

\subsection{Robotics}\label{sec:brax}
Our final experimental domain is robotics, using Brax~\citep{todorov2012Mujoco,brax2021Github} and a horizon of 100 timesteps. 
There are two types of levels that the adversary can choose, type A or type B. In type A levels, the task is standard: to walk forward. The levels represent disturbances to the agent's actions; the agent's action $a\in \mathbb{R}^n$ is multiplied by the level $\theta \in \mathbb{R}^n$ before being applied to the simulation.
Type B levels require the agent to perform a specific, but arbitrary, action that it cannot observe. The agent receives a large positive or negative reward, depending on whether the action is correct or not, and the episode terminates. After training, we evaluate performance when no disturbances are applied.
For these experiments, we use an episode length of 100, and consider four separate morphologies, Ant, Half-Cheetah, Hopper and Walker.

We present the results in \cref{tab:brax_results}. MMR and minimax return both fall victim to irreducible regret, training primarily on type B levels, causing the agent to fail when presented with type A ones. \remidi, by contrast, does not suffer from this problem, and is able to learn an effective policy.

\begin{table}[H]
    \caption{Brax results; mean $\pm$ standard deviation over 10 seeds.}
    \label{tab:brax_results}
    \resizebox{1\linewidth}{!}{\begin{tabular}{lllll}
\toprule
Method & Ant & Cheetah & Hopper & Walker \\
\midrule
$\text{PLR}^\perp$ & -50 $\pm$ 2 & -19 $\pm$ 14 & 67 $\pm$ 19 & 126 $\pm$ 68 \\
Minimax Return & -49 $\pm$ 4 & -16 $\pm$ 3 & 86 $\pm$ 15 & 148 $\pm$ 67 \\
ReMiDi & \textbf{257} $\pm$ 35 & \textbf{321} $\pm$ 16 & \textbf{193} $\pm$ 2 & \textbf{344} $\pm$ 21 \\
\bottomrule
\end{tabular}
}
\end{table}




\section{Related Work}\label{sec:related_work}

\textbf{UED and Adaptive Curricula.}
Many recent works aim to find an adaptive curriculum for an RL agent to train on, with different methods using different metrics to choose training levels.
The learning potential~\citep{oudeyer2007Intrinsic} of an agent on a particular level or training example is a measure of how much an agent's loss or reward will improve after training on this data point. Therefore, a level with a high learning potential is a promising one to train on, as the agent can still improve on it, but the level is not too difficult for the agent's current capability. While methods based on learning potential have shown promise empirically~\citep{florensa2017Automatic,portelas2019Teacher,matiisen2017Teacherstudent}, techniques that aim for robustness adversarially sample levels; for instance, adversarial minimax trains the agent in levels that minimise the agent's performance~\citep{pinto2017Robust,wang2019POET}. 
However, these worst-case levels may be impossible, and therefore provide no learning benefit~\citep{dennis2020Emergent}.
Minimax regret provides stronger robustness guarantees~\citep{dennis2020Emergent,jiang2021Replayguided,holder2022Evolving} and also alleviates the problem of minimax choosing impossible levels. 
Implicit in these methods is that regret corresponds to some notion of learnability~\citep{dennis2020Emergent}.
However, we have shown that regret does not always coincide with learnability, and that minimax regret may therefore cause learning to stagnate, a problem ReMiDi addresses. 
Some works in UED also discuss problems with minimax regret; for instance, \citet{gur2022Environment} describe situations in which a minimax regret adversary is incentivised to generate levels that are too difficult for the agent, and would provide no learning. 
\citet{garcin2024Dred} also show that minimax regret may not be a suitable objective if there is a particular distribution of interest we desire the agent to perform well over. They address this problem by using a small sample from this distribution to inform the level-generation process.

\textbf{Decision Theory.}
From another perspective, regret-based UED implements a minimax regret decision rule---i.e., minimising the worst-case regret~\citep{savage1951Theory,luce1957Games,peterson2017Introduction,dennis2020Emergent}. 
This decision rule provides an ordering over policies, preferring policies with lower worst-case regret. However, any two policies that have the same worst-case regret are treated as equally good by this decision rule. 
Our BLP objective imposes further ordering, breaking ties according to worst-case regret over all other levels that are distinguishable from MMR levels.

\textbf{Game Theory and Equilibrium Refinements.}
Our work is also related to the rich literature of Nash equilibrium refinements in game theory~\citep{selten1975Spieltheoretische,selten1975Reexamination,kreps1982Sequential,fudenberg1991Perfect,osborne2004Introduction,bonanno2015Game}.
In particular, our BLP solution concept is similar to the notion of Perfect Bayesian (or sequential) equilibria~\citep{selten1975Reexamination,kreps1982Sequential,fudenberg1991Perfect}, where a strategy must be optimal at all nodes of the game tree (given some belief), regardless of whether these nodes are reached under equilibrium or not.

\textbf{SAMPLR.} \citet{jiang2022Grounding} also focus on the limitations of standard minimax regret. They identify the problem of an automatic curriculum shifting the distribution of environment parameters away from the ground truth, which is problematic in partially-observable settings. 
The authors propose a solution to this problem that involves training on fictitious transitions based on the ground-truth distribution of unseen parameters. 
The regret stagnation problem we address---where degenerate environments with irreducible regret are continually trained on---is also a side-effect of using minimax regret in partially observable settings.
Importantly, one of SAMPLR's core assumptions is that the adversary and ground truth share support; however, as we have shown, this is not a given, as the adversary can place all of its probability mass on irreducible regret levels.

\textbf{Worst-Case and Robust RL.}
There are also other formalisms outside of the UED field, such as robust MDPs and robust control in general~\citep{iyengar2005Robust,el2005Robust,xu2010Distributionally,wiesemann2013Robust,lim2013Reinforcement,shen2020Deep,nakao2019Distributionally,yang2022Rorl,bauerle2020Distributionally,wang2023Foundation,bhardwaj2023Adversarial,ye2023Corruptionrobust,goyal2018Robust}. This field generally has a different overall objective to UED. In particular, the Robust MDP literature often focuses on the case where the model of an MDP has some statistical or estimation errors, and the task is to learn a robust policy with respect to the real dynamics. 
UPOMDPs, on the other hand, are rather aimed at modelling transfer as a decision under ignorance problem~\citep{milnor1951Games,peterson2017Introduction,dennis2020Emergent}, where there is extreme uncertainty about the environment, and the goal is to obtain an agent that obtains minimax regret over this set of environments. 
Furthermore, Robust RL methods tend to focus on worst-case objectives~\citep{iyengar2005Robust,el2005Robust} which---as our results and prior work~\cite{dennis2020Emergent} have shown---often prioritises unsolvable levels.

\section{Limitations \& Future Work}

One limitation of our work---indeed, of the UED field in general---is that all of our theoretical results hold only at convergence to Nash equilibrium. While a Nash equilibrium always exists in finite games~\citep{nash1950Equilibrium}, there is no guarantee that we will converge to it. Despite this, UED methods are still empirically useful~\citep{dennis2020Emergent,holder2022Evolving,samvelyan2023Maestro,team2023Humantimescale}.
Using more sophisticated algorithms for the agent and adversary---such as magnetic mirror descent~\citep{sokota2022Unified}---could be a way to alleviate this problem.

Furthermore, since we build on prior UED work, we use the same set of assumptions, that is, that our UPOMDPs are finite and discrete. In a continuous setting, we would require some other condition that ensures equilibria exist (for instance, closed convex continuous strategy spaces). 
Despite this, our algorithm can directly be applied to continuous settings, as our robotics experiments have shown.

We also note that \remidi requires finding the Nash equilibrium of multiple MMR games, which is theoretically more computationally challenging than just finding the equilibrium of one game. However, since existing UED methods merely run for a fixed number of environment timesteps, this problem falls away in practice. Furthermore, to obtain a BLP policy, \remidi potentially has to perform a large number of outer loop iterations (at most $|\Theta|$). However, in practice, we similarly approximate this by running for a small (e.g. 4) number of outer-loop iterations.

We intentionally designed our experiments to exhibit irreducible regret, but we note that in many prior benchmarks, learning stagnation due to irreducible regret levels has not been a major problem. One reason for this is if a policy can do well on all levels, then irreducible regret does not occur. Another reason is that irreducible regret is less of a problem when using imperfect measures of regret, as current methods do~\citep{dennis2020Emergent,jiang2021Replayguided}. 
Nevertheless, as our regret measures improve, and as UED is applied to more complex environments with more room for irreducible regret, our analysis will become increasingly relevant. We would further like to extend our work to the more general case where the score of a level does not correspond to its \textit{learnability}. Moreover, we believe that the connection (and difference) between regret and learnability merits further investigation. This would improve understanding of, and potentially disentangle the separate robustness~\citep{iyengar2005Robust,pinto2017Robust} and curriculum~\citep{matiisen2017Teacherstudent,florensa2017Automatic} aspects of UED.

A final limitation is that our algorithm requires knowing the likelihood of a trajectory under a set of environments a priori, but this may be impractical in practice. Learning a belief model~\citep{hu2021OffBelief,jiang2022Grounding} could alleviate this issue and make our algorithm more practical. Relatedly, we would like to investigate more practical ways of not degrading performance on previous adversaries, for instance, periodically training on these prior buffers.

There are several promising avenues for future work, such as attempting to refine our solution concept further, to move towards true leximinimax regret~\citep{peterson2017Introduction}. Ties among minimax regret policies would then be broken by considering the second-worst case environments, and so on (instead of considering only the second-worst case levels that are \textit{distinguishable} from MMR levels). 
Another avenue that would merit further investigation is using constrained optimisation approaches (e.g. \citet{ghosh2022provably}), where we train the policy on new levels with the constraint that its performance on past levels cannot degrade.

\section{Conclusion}
In this work, we show that the minimax regret decision rule has a significant limitation in partially-observable settings with high irreducible regret: the minimax regret policy may perform poorly on a large subset of levels. We propose the Bayesian level-perfect objective, a refinement of minimax regret that iteratively solves smaller games, obtaining minimax regret-like performance over more environments than is guaranteed by regret-based UED, alleviating this problem of MMR. We also develop a proof-of-concept algorithm, ReMiDi, that results in a BLP policy at convergence.

We provide theoretical justification for our approach, and show that it retains minimax regret guarantees while monotonically improving worst-case regret on non-highest-regret levels.
We perform experiments in several settings with high irreducible regret and demonstrate that regret-based UED methods suffer from the regret stagnation problem we describe, causing the resulting policies to perform poorly on a large number of levels.
ReMiDi, by contrast, is less susceptible to environment distributions that contain degenerate high-irreducible regret environments.

Ultimately, we believe that UED has great potential to develop robust, generalisable and effective RL agents. 
Our BLP solution concept is an improved objective that obtains stronger robustness guarantees than minimax regret, and avoids its pathologies in irreducible-regret environments---paving the way for UED to be applied to more open-ended and complex task spaces.
\section*{Acknowledgements}
MB is supported by the Rhodes Trust. SC is funded by the EPSRC Centre for
Doctoral Training in Autonomous Intelligent Machines and Systems.
This work was supported by UK Research and Innovation and the European Research Council, selected by the ERC, and funded by the UKRI [grant number EP/Y028481/1].
\section*{Impact Statement}
This paper presents work whose goal is to advance the field of Machine Learning. There are many potential societal consequences of our work, none of which we feel must be specifically highlighted here.


\bibliography{bib.bib}

\begin{thebibliography}{57}
\providecommand{\natexlab}[1]{#1}
\providecommand{\url}[1]{\texttt{#1}}
\expandafter\ifx\csname urlstyle\endcsname\relax
  \providecommand{\doi}[1]{doi: #1}\else
  \providecommand{\doi}{doi: \begingroup \urlstyle{rm}\Url}\fi

\bibitem[B{\"{a}}uerle \& Glauner(2022)B{\"{a}}uerle and
  Glauner]{bauerle2020Distributionally}
B{\"{a}}uerle, N. and Glauner, A.
\newblock Distributionally robust markov decision processes and their
  connection to risk measures.
\newblock \emph{Math. Oper. Res.}, 47\penalty0 (3):\penalty0 1757--1780, 2022.
\newblock \doi{10.1287/MOOR.2021.1187}.
\newblock URL \url{https://doi.org/10.1287/moor.2021.1187}.

\bibitem[Bell(1982)]{Bell82}
Bell, D.~E.
\newblock Regret in decision making under uncertainty.
\newblock \emph{Operations Research}, 30\penalty0 (5):\penalty0 961--981, 1982.
\newblock ISSN 0030364X, 15265463.
\newblock URL \url{http://www.jstor.org/stable/170353}.

\bibitem[Bhardwaj et~al.(2023)Bhardwaj, Xie, Boots, Jiang, and
  Cheng]{bhardwaj2023Adversarial}
Bhardwaj, M., Xie, T., Boots, B., Jiang, N., and Cheng, C.
\newblock Adversarial model for offline reinforcement learning.
\newblock In Oh, A., Naumann, T., Globerson, A., Saenko, K., Hardt, M., and
  Levine, S. (eds.), \emph{Advances in Neural Information Processing Systems
  36: Annual Conference on Neural Information Processing Systems 2023, NeurIPS
  2023, New Orleans, LA, USA, December 10 - 16, 2023}, 2023.
\newblock URL
  \url{http://papers.nips.cc/paper\_files/paper/2023/hash/0429ececfb199efc93182990169e73bb-Abstract-Conference.html}.

\bibitem[Bonanno(2015)]{bonanno2015Game}
Bonanno, G.
\newblock Game theory: Parts i and ii-with 88 solved exercises. an open access
  textbook.
\newblock Technical report, Working Paper, 2015.

\bibitem[Coward et~al.(2024)Coward, Beukman, and Foerster]{coward2024JaxUED}
Coward, S., Beukman, M., and Foerster, J.
\newblock Jaxued: A simple and useable ued library in jax.
\newblock \emph{arXiv preprint}, 2024.

\bibitem[Dennis et~al.(2020)Dennis, Jaques, Vinitsky, Bayen, Russell, Critch,
  and Levine]{dennis2020Emergent}
Dennis, M., Jaques, N., Vinitsky, E., Bayen, A.~M., Russell, S., Critch, A.,
  and Levine, S.
\newblock Emergent complexity and zero-shot transfer via unsupervised
  environment design.
\newblock In Larochelle, H., Ranzato, M., Hadsell, R., Balcan, M., and Lin, H.
  (eds.), \emph{Advances in Neural Information Processing Systems}, 2020.
\newblock URL \url{https://proceedings.neurips.cc/paper/2020/hash/
  985e9a46e10005356bbaf194249f6856-Abstract.html}.

\bibitem[El~Ghaoui \& Nilim(2005)El~Ghaoui and Nilim]{el2005Robust}
El~Ghaoui, L. and Nilim, A.
\newblock Robust solutions to markov decision problems with uncertain
  transition matrices.
\newblock \emph{Operations Research}, 53\penalty0 (5):\penalty0 780--798, 2005.

\bibitem[Florensa et~al.(2018)Florensa, Held, Geng, and
  Abbeel]{florensa2017Automatic}
Florensa, C., Held, D., Geng, X., and Abbeel, P.
\newblock Automatic goal generation for reinforcement learning agents.
\newblock In Dy, J.~G. and Krause, A. (eds.), \emph{Proceedings of the 35th
  International Conference on Machine Learning, {ICML} 2018,
  Stockholmsm{\"{a}}ssan, Stockholm, Sweden, July 10-15, 2018}, volume~80 of
  \emph{Proceedings of Machine Learning Research}, pp.\  1514--1523. {PMLR},
  2018.
\newblock URL \url{http://proceedings.mlr.press/v80/florensa18a.html}.

\bibitem[Freeman et~al.(2021)Freeman, Frey, Raichuk, Girgin, Mordatch, and
  Bachem]{brax2021Github}
Freeman, C.~D., Frey, E., Raichuk, A., Girgin, S., Mordatch, I., and Bachem, O.
\newblock Brax - a differentiable physics engine for large scale rigid body
  simulation, 2021.
\newblock URL \url{http://github.com/google/brax}.

\bibitem[Fudenberg \& Tirole(1991)Fudenberg and Tirole]{fudenberg1991Perfect}
Fudenberg, D. and Tirole, J.
\newblock Perfect bayesian equilibrium and sequential equilibrium.
\newblock \emph{journal of Economic Theory}, 53\penalty0 (2):\penalty0
  236--260, 1991.

\bibitem[Garcin et~al.(2024)Garcin, Doran, Guo, Lucas, and
  Albrecht]{garcin2024Dred}
Garcin, S., Doran, J., Guo, S., Lucas, C.~G., and Albrecht, S.~V.
\newblock Dred: Zero-shot transfer in reinforcement learning via
  data-regularised environment design.
\newblock 2024.
\newblock URL \url{https://doi.org/10.48550/arXiv.2402.03479}.

\bibitem[Ghosh et~al.(2022)Ghosh, Zhou, and Shroff]{ghosh2022provably}
Ghosh, A., Zhou, X., and Shroff, N.
\newblock Provably efficient model-free constrained rl with linear function
  approximation.
\newblock \emph{Advances in Neural Information Processing Systems},
  35:\penalty0 13303--13315, 2022.

\bibitem[Goyal \& Grand{-}Cl{\'{e}}ment(2023)Goyal and
  Grand{-}Cl{\'{e}}ment]{goyal2018Robust}
Goyal, V. and Grand{-}Cl{\'{e}}ment, J.
\newblock Robust markov decision processes: Beyond rectangularity.
\newblock \emph{Math. Oper. Res.}, 48\penalty0 (1):\penalty0 203--226, 2023.
\newblock \doi{10.1287/MOOR.2022.1259}.
\newblock URL \url{https://doi.org/10.1287/moor.2022.1259}.

\bibitem[Gur et~al.(2021)Gur, Jaques, Miao, Choi, Tiwari, Lee, and
  Faust]{gur2022Environment}
Gur, I., Jaques, N., Miao, Y., Choi, J., Tiwari, M., Lee, H., and Faust, A.
\newblock Environment generation for zero-shot compositional reinforcement
  learning.
\newblock In Ranzato, M., Beygelzimer, A., Dauphin, Y.~N., Liang, P., and
  Vaughan, J.~W. (eds.), \emph{Advances in Neural Information Processing
  Systems}, pp.\  4157--4169, 2021.
\newblock URL
  \url{https://proceedings.neurips.cc/paper/2021/hash/218344619d8fb95d504ccfa11804073f-Abstract.html}.

\bibitem[Hochreiter \& Schmidhuber(1997)Hochreiter and
  Schmidhuber]{hochreiter1997Long}
Hochreiter, S. and Schmidhuber, J.
\newblock Long short-term memory.
\newblock \emph{Neural Comput.}, 9\penalty0 (8):\penalty0 1735--1780, 1997.
\newblock \doi{10.1162/neco.1997.9.8.1735}.
\newblock URL \url{https://doi.org/10.1162/neco.1997.9.8.1735}.

\bibitem[Hu et~al.(2020)Hu, Lerer, Peysakhovich, and Foerster]{hu2020Otherplay}
Hu, H., Lerer, A., Peysakhovich, A., and Foerster, J.~N.
\newblock "other-play" for zero-shot coordination.
\newblock In \emph{Proceedings of the 37th International Conference on Machine
  Learning, {ICML} 2020, 13-18 July 2020, Virtual Event}, volume 119 of
  \emph{Proceedings of Machine Learning Research}, pp.\  4399--4410. {PMLR},
  2020.
\newblock URL \url{http://proceedings.mlr.press/v119/hu20a.html}.

\bibitem[Hu et~al.(2021)Hu, Lerer, Cui, Pineda, Brown, and
  Foerster]{hu2021OffBelief}
Hu, H., Lerer, A., Cui, B., Pineda, L., Brown, N., and Foerster, J.~N.
\newblock Off-belief learning.
\newblock In Meila, M. and Zhang, T. (eds.), \emph{Proceedings of the 38th
  International Conference on Machine Learning}, volume 139 of
  \emph{Proceedings of Machine Learning Research}, pp.\  4369--4379. {PMLR},
  2021.
\newblock URL \url{http://proceedings.mlr.press/v139/hu21c.html}.

\bibitem[Iyengar(2005)]{iyengar2005Robust}
Iyengar, G.~N.
\newblock Robust dynamic programming.
\newblock \emph{Mathematics of Operations Research}, 30\penalty0 (2):\penalty0
  257--280, 2005.

\bibitem[Jiang et~al.(2021{\natexlab{a}})Jiang, Dennis, Parker{-}Holder,
  Foerster, Grefenstette, and Rockt{\"{a}}schel]{jiang2021Replayguided}
Jiang, M., Dennis, M., Parker{-}Holder, J., Foerster, J.~N., Grefenstette, E.,
  and Rockt{\"{a}}schel, T.
\newblock Replay-guided adversarial environment design.
\newblock In Ranzato, M., Beygelzimer, A., Dauphin, Y.~N., Liang, P., and
  Vaughan, J.~W. (eds.), \emph{Advances in Neural Information Processing
  Systems}, pp.\  1884--1897, 2021{\natexlab{a}}.
\newblock URL \url{https://proceedings.neurips.cc/paper/2021/hash/
  0e915db6326b6fb6a3c56546980a8c93-Abstract.html}.

\bibitem[Jiang et~al.(2021{\natexlab{b}})Jiang, Grefenstette, and
  Rockt{\"{a}}schel]{jiang2020Prioritized}
Jiang, M., Grefenstette, E., and Rockt{\"{a}}schel, T.
\newblock Prioritized level replay.
\newblock In Meila, M. and Zhang, T. (eds.), \emph{Proceedings of the 38th
  International Conference on Machine Learning}, volume 139, pp.\  4940--4950.
  {PMLR}, 2021{\natexlab{b}}.
\newblock URL \url{http://proceedings.mlr.press/v139/jiang21b.html}.

\bibitem[Jiang et~al.(2022)Jiang, Dennis, Parker-Holder, Lupu, K{\"u}ttler,
  Grefenstette, Rockt{\"a}schel, and Foerster]{jiang2022Grounding}
Jiang, M., Dennis, M., Parker-Holder, J., Lupu, A., K{\"u}ttler, H.,
  Grefenstette, E., Rockt{\"a}schel, T., and Foerster, J.
\newblock Grounding aleatoric uncertainty for unsupervised environment design.
\newblock \emph{Advances in Neural Information Processing Systems},
  35:\penalty0 32868--32881, 2022.

\bibitem[Jiang et~al.(2023)Jiang, Dennis, Grefenstette, and
  Rocktäschel]{jiang2023Minimax}
Jiang, M., Dennis, M., Grefenstette, E., and Rocktäschel, T.
\newblock minimax: Efficient baselines for autocurricula in jax.
\newblock In \emph{Agent Learning in Open-Endedness Workshop at NeurIPS}, 2023.

\bibitem[Kreps \& Wilson(1982)Kreps and Wilson]{kreps1982Sequential}
Kreps, D.~M. and Wilson, R.
\newblock Sequential equilibria.
\newblock \emph{Econometrica: Journal of the Econometric Society}, pp.\
  863--894, 1982.

\bibitem[Lim et~al.(2013)Lim, Xu, and Mannor]{lim2013Reinforcement}
Lim, S.~H., Xu, H., and Mannor, S.
\newblock Reinforcement learning in robust markov decision processes.
\newblock \emph{Advances in Neural Information Processing Systems}, 26, 2013.

\bibitem[Loomes \& Sugden(1982)Loomes and Sugden]{Loomes82}
Loomes, G. and Sugden, R.
\newblock Regret theory: An alternative theory of rational choice under
  uncertainty.
\newblock \emph{The Economic Journal}, 92\penalty0 (368):\penalty0 805--824,
  1982.
\newblock ISSN 00130133, 14680297.
\newblock URL \url{http://www.jstor.org/stable/2232669}.

\bibitem[Lu et~al.(2022)Lu, Kuba, Letcher, Metz, Schroeder~de Witt, and
  Foerster]{lu2022Discovered}
Lu, C., Kuba, J., Letcher, A., Metz, L., Schroeder~de Witt, C., and Foerster,
  J.
\newblock Discovered policy optimisation.
\newblock \emph{Advances in Neural Information Processing Systems},
  35:\penalty0 16455--16468, 2022.

\bibitem[Luce \& Raiffa(1957)Luce and Raiffa]{luce1957Games}
Luce, R.~D. and Raiffa, H.
\newblock \emph{Games and Decisions: Introduction and Critical Survey}.
\newblock Wiley, New York, 1957.

\bibitem[Matiisen et~al.(2020)Matiisen, Oliver, Cohen, and
  Schulman]{matiisen2017Teacherstudent}
Matiisen, T., Oliver, A., Cohen, T., and Schulman, J.
\newblock Teacher-student curriculum learning.
\newblock volume~31, pp.\  3732--3740, 2020.
\newblock \doi{10.1109/TNNLS.2019.2934906}.
\newblock URL \url{https://doi.org/10.1109/TNNLS.2019.2934906}.

\bibitem[Mediratta et~al.(2023)Mediratta, Jiang, Parker-Holder, Dennis,
  Vinitsky, and Rockt{\"a}schel]{mediratta2023Stabilizing}
Mediratta, I., Jiang, M., Parker-Holder, J., Dennis, M., Vinitsky, E., and
  Rockt{\"a}schel, T.
\newblock Stabilizing unsupervised environment design with a learned adversary.
\newblock In \emph{Conference on Lifelong Learning Agents}, pp.\  270--291,
  2023.

\bibitem[Milnor(1951)]{milnor1951Games}
Milnor, J.
\newblock Games against nature.
\newblock Technical report, RAND PROJECT AIR FORCE SANTA MONICA CA, 1951.

\bibitem[Nakao et~al.(2021)Nakao, Jiang, and Shen]{nakao2019Distributionally}
Nakao, H., Jiang, R., and Shen, S.
\newblock Distributionally robust partially observable markov decision process
  with moment-based ambiguity.
\newblock \emph{{SIAM} J. Optim.}, 31\penalty0 (1):\penalty0 461--488, 2021.
\newblock \doi{10.1137/19M1268410}.
\newblock URL \url{https://doi.org/10.1137/19M1268410}.

\bibitem[Nash(1950)]{nash1950Equilibrium}
Nash, J.~F.
\newblock Equilibrium points in n-person games.
\newblock \emph{Proceedings of the national academy of sciences}, 36\penalty0
  (1):\penalty0 48--49, 1950.

\bibitem[Osborne \& Rubinstein(1994)Osborne and Rubinstein]{osborne1994Course}
Osborne, M.~J. and Rubinstein, A.
\newblock \emph{A course in game theory}.
\newblock MIT press, 1994.

\bibitem[Osborne et~al.(2004)]{osborne2004Introduction}
Osborne, M.~J. et~al.
\newblock \emph{An introduction to game theory}, volume~3.
\newblock Oxford university press New York, 2004.

\bibitem[Oudeyer et~al.(2007)Oudeyer, Kaplan, and Hafner]{oudeyer2007Intrinsic}
Oudeyer, P.-Y., Kaplan, F., and Hafner, V.~V.
\newblock Intrinsic motivation systems for autonomous mental development.
\newblock \emph{IEEE transactions on evolutionary computation}, 11\penalty0
  (2):\penalty0 265--286, 2007.

\bibitem[Parker{-}Holder et~al.(2022)Parker{-}Holder, Jiang, Dennis, Samvelyan,
  Foerster, Grefenstette, and Rockt{\"{a}}schel]{holder2022Evolving}
Parker{-}Holder, J., Jiang, M., Dennis, M., Samvelyan, M., Foerster, J.,
  Grefenstette, E., and Rockt{\"{a}}schel, T.
\newblock Evolving curricula with regret-based environment design.
\newblock In \emph{Proceedings of the International Conference on Machine
  Learning}, pp.\  17473--17498. {PMLR}, 2022.
\newblock URL \url{https://proceedings.mlr.press/v162/parker-holder22a.html}.

\bibitem[Peterson(2017)]{peterson2017Introduction}
Peterson, M.
\newblock \emph{An introduction to decision theory}.
\newblock Cambridge University Press, 2017.

\bibitem[Pinto et~al.(2017)Pinto, Davidson, Sukthankar, and
  Gupta]{pinto2017Robust}
Pinto, L., Davidson, J., Sukthankar, R., and Gupta, A.
\newblock Robust adversarial reinforcement learning.
\newblock In Precup, D. and Teh, Y.~W. (eds.), \emph{Proceedings of the 34th
  International Conference on Machine Learning}, volume~70 of \emph{Proceedings
  of Machine Learning Research}, pp.\  2817--2826. PMLR, 06--11 Aug 2017.
\newblock URL \url{https://proceedings.mlr.press/v70/pinto17a.html}.

\bibitem[Portelas et~al.(2019)Portelas, Colas, Hofmann, and
  Oudeyer]{portelas2019Teacher}
Portelas, R., Colas, C., Hofmann, K., and Oudeyer, P.
\newblock Teacher algorithms for curriculum learning of deep {RL} in
  continuously parameterized environments.
\newblock In Kaelbling, L.~P., Kragic, D., and Sugiura, K. (eds.), \emph{3rd
  Annual Conference on Robot Learning, CoRL 2019, Osaka, Japan, October 30 -
  November 1, 2019, Proceedings}, volume 100 of \emph{Proceedings of Machine
  Learning Research}, pp.\  835--853. {PMLR}, 2019.
\newblock URL \url{http://proceedings.mlr.press/v100/portelas20a.html}.

\bibitem[Samvelyan et~al.(2023)Samvelyan, Khan, Dennis, Jiang, Parker{-}Holder,
  Foerster, Raileanu, and Rockt{\"{a}}schel]{samvelyan2023Maestro}
Samvelyan, M., Khan, A., Dennis, M., Jiang, M., Parker{-}Holder, J., Foerster,
  J.~N., Raileanu, R., and Rockt{\"{a}}schel, T.
\newblock {MAESTRO:} open-ended environment design for multi-agent
  reinforcement learning.
\newblock In \emph{The Eleventh International Conference on Learning
  Representations, {ICLR} 2023, Kigali, Rwanda, May 1-5, 2023}. OpenReview.net,
  2023.
\newblock URL \url{https://openreview.net/pdf?id=sKWlRDzPfd7}.

\bibitem[Savage(1951)]{savage1951Theory}
Savage, L.~J.
\newblock The theory of statistical decision.
\newblock \emph{Journal of the American Statistical association}, 46\penalty0
  (253):\penalty0 55--67, 1951.

\bibitem[Selten(1975{\natexlab{a}})]{selten1975Reexamination}
Selten, R.
\newblock Reexamination of the perfectness concept for equilibrium points in
  extensive games.
\newblock \emph{International Journal of Game Theory}, 4, 1975{\natexlab{a}}.

\bibitem[Selten(1975{\natexlab{b}})]{selten1975Spieltheoretische}
Selten, R.
\newblock Spieltheoretische behandlung eines oligopolmodells mit
  nachfragetr{\"a}gheit.
\newblock \emph{Zeitschrift f{\"u}r die gesamte Staatswissenschaft/Journal of
  Institutional and Theoretical Economics}, \penalty0 (H. 2):\penalty0
  374--374, 1975{\natexlab{b}}.

\bibitem[Shen et~al.(2020)Shen, Li, Jiang, Wang, and Zhao]{shen2020Deep}
Shen, Q., Li, Y., Jiang, H., Wang, Z., and Zhao, T.
\newblock Deep reinforcement learning with robust and smooth policy.
\newblock In \emph{Proceedings of the 37th International Conference on Machine
  Learning, {ICML} 2020, 13-18 July 2020, Virtual Event}, volume 119 of
  \emph{Proceedings of Machine Learning Research}, pp.\  8707--8718. {PMLR},
  2020.
\newblock URL \url{http://proceedings.mlr.press/v119/shen20b.html}.

\bibitem[Sokota et~al.(2023)Sokota, D'Orazio, Kolter, Loizou, Lanctot,
  Mitliagkas, Brown, and Kroer]{sokota2022Unified}
Sokota, S., D'Orazio, R., Kolter, J.~Z., Loizou, N., Lanctot, M., Mitliagkas,
  I., Brown, N., and Kroer, C.
\newblock A unified approach to reinforcement learning, quantal response
  equilibria, and two-player zero-sum games.
\newblock In \emph{The Eleventh International Conference on Learning
  Representations, {ICLR} 2023, Kigali, Rwanda, May 1-5, 2023}. OpenReview.net,
  2023.
\newblock URL \url{https://openreview.net/pdf?id=DpE5UYUQzZH}.

\bibitem[Sukhbaatar et~al.(2018)Sukhbaatar, Lin, Kostrikov, Synnaeve, Szlam,
  and Fergus]{sukhbaatar2017Intrinsic}
Sukhbaatar, S., Lin, Z., Kostrikov, I., Synnaeve, G., Szlam, A., and Fergus, R.
\newblock Intrinsic motivation and automatic curricula via asymmetric
  self-play.
\newblock In \emph{6th International Conference on Learning Representations}.
  OpenReview.net, 2018.
\newblock URL \url{https://openreview.net/forum?id=SkT5Yg-RZ}.

\bibitem[Sutton \& Barto(2018)Sutton and Barto]{sutton2018reinforcement}
Sutton, R.~S. and Barto, A.~G.
\newblock \emph{Reinforcement learning: An introduction}.
\newblock MIT press, 2018.

\bibitem[Team et~al.(2023)Team, Bauer, Baumli, Baveja, Behbahani, Bhoopchand,
  Bradley{-}Schmieg, Chang, Clay, Collister, Dasagi, Gonzalez, Gregor, Hughes,
  Kashem, Loks{-}Thompson, Openshaw, Parker{-}Holder, Pathak, Nieves,
  Rakicevic, Rockt{\"{a}}schel, Schroecker, Sygnowski, Tuyls, York, Zacherl,
  and Zhang]{team2023Humantimescale}
Team, A.~A., Bauer, J., Baumli, K., Baveja, S., Behbahani, F. M.~P.,
  Bhoopchand, A., Bradley{-}Schmieg, N., Chang, M., Clay, N., Collister, A.,
  Dasagi, V., Gonzalez, L., Gregor, K., Hughes, E., Kashem, S.,
  Loks{-}Thompson, M., Openshaw, H., Parker{-}Holder, J., Pathak, S., Nieves,
  N.~P., Rakicevic, N., Rockt{\"{a}}schel, T., Schroecker, Y., Sygnowski, J.,
  Tuyls, K., York, S., Zacherl, A., and Zhang, L.
\newblock Human-timescale adaptation in an open-ended task space.
\newblock \emph{CoRR}, abs/2301.07608, 2023.
\newblock \doi{10.48550/arXiv.2301.07608}.
\newblock URL \url{https://doi.org/10.48550/arXiv.2301.07608}.

\bibitem[Tobin et~al.(2017)Tobin, Fong, Ray, Schneider, Zaremba, and
  Abbeel]{tobin2017Domain}
Tobin, J., Fong, R., Ray, A., Schneider, J., Zaremba, W., and Abbeel, P.
\newblock Domain randomization for transferring deep neural networks from
  simulation to the real world.
\newblock In \emph{International Conference on Intelligent Robots and Systems},
  pp.\  23--30. {IEEE}, 2017.
\newblock \doi{10.1109/IROS.2017.8202133}.
\newblock URL \url{https://doi.org/10.1109/IROS.2017.8202133}.

\bibitem[Todorov et~al.(2012)Todorov, Erez, and Tassa]{todorov2012Mujoco}
Todorov, E., Erez, T., and Tassa, Y.
\newblock Mujoco: {A} physics engine for model-based control.
\newblock In \emph{International Conference on Intelligent Robots and Systems},
  pp.\  5026--5033. {IEEE}, 2012.
\newblock \doi{10.1109/IROS.2012.6386109}.
\newblock URL \url{https://doi.org/10.1109/IROS.2012.6386109}.

\bibitem[Wang et~al.(2019)Wang, Lehman, Clune, and Stanley]{wang2019POET}
Wang, R., Lehman, J., Clune, J., and Stanley, K.~O.
\newblock {Paired Open-Ended Trailblazer} {(POET):} {E}ndlessly generating
  increasingly complex and diverse learning environments and their solutions.
\newblock \emph{CoRR}, abs/1901.01753, 2019.
\newblock URL \url{http://arxiv.org/abs/1901.01753}.

\bibitem[Wang et~al.(2023)Wang, Si, Blanchet, and Zhou]{wang2023Foundation}
Wang, S., Si, N., Blanchet, J.~H., and Zhou, Z.
\newblock On the foundation of distributionally robust reinforcement learning.
\newblock \emph{CoRR}, abs/2311.09018, 2023.
\newblock \doi{10.48550/ARXIV.2311.09018}.
\newblock URL \url{https://doi.org/10.48550/arXiv.2311.09018}.

\bibitem[Wiesemann et~al.(2013)Wiesemann, Kuhn, and
  Rustem]{wiesemann2013Robust}
Wiesemann, W., Kuhn, D., and Rustem, B.
\newblock Robust markov decision processes.
\newblock \emph{Mathematics of Operations Research}, 38\penalty0 (1):\penalty0
  153--183, 2013.

\bibitem[Wijmans et~al.(2023)Wijmans, Savva, Essa, Lee, Morcos, and
  Batra]{wijmans2023Emergence}
Wijmans, E., Savva, M., Essa, I., Lee, S., Morcos, A.~S., and Batra, D.
\newblock Emergence of maps in the memories of blind navigation agents.
\newblock \emph{{AI} Matters}, 9\penalty0 (2):\penalty0 8--14, 2023.
\newblock \doi{10.1145/3609468.3609471}.
\newblock URL \url{https://doi.org/10.1145/3609468.3609471}.

\bibitem[Xu \& Mannor(2010)Xu and Mannor]{xu2010Distributionally}
Xu, H. and Mannor, S.
\newblock Distributionally robust markov decision processes.
\newblock In Lafferty, J.~D., Williams, C. K.~I., Shawe{-}Taylor, J., Zemel,
  R.~S., and Culotta, A. (eds.), \emph{Advances in Neural Information
  Processing Systems 23: 24th Annual Conference on Neural Information
  Processing Systems 2010. Proceedings of a meeting held 6-9 December 2010,
  Vancouver, British Columbia, Canada}, pp.\  2505--2513. Curran Associates,
  Inc., 2010.
\newblock URL
  \url{https://proceedings.neurips.cc/paper/2010/hash/19f3cd308f1455b3fa09a282e0d496f4-Abstract.html}.

\bibitem[Yang et~al.(2022)Yang, Bai, Ma, Wang, Zhang, and Han]{yang2022Rorl}
Yang, R., Bai, C., Ma, X., Wang, Z., Zhang, C., and Han, L.
\newblock {RORL:} robust offline reinforcement learning via conservative
  smoothing.
\newblock In Koyejo, S., Mohamed, S., Agarwal, A., Belgrave, D., Cho, K., and
  Oh, A. (eds.), \emph{Advances in Neural Information Processing Systems 35:
  Annual Conference on Neural Information Processing Systems 2022, NeurIPS
  2022, New Orleans, LA, USA, November 28 - December 9, 2022}, 2022.
\newblock URL
  \url{http://papers.nips.cc/paper\_files/paper/2022/hash/96bbdd0ed2a9e7cd2fb7caf2fae15f3d-Abstract-Conference.html}.

\bibitem[Ye et~al.(2023)Ye, Yang, Gu, and Zhang]{ye2023Corruptionrobust}
Ye, C., Yang, R., Gu, Q., and Zhang, T.
\newblock Corruption-robust offline reinforcement learning with general
  function approximation.
\newblock In Oh, A., Naumann, T., Globerson, A., Saenko, K., Hardt, M., and
  Levine, S. (eds.), \emph{Advances in Neural Information Processing Systems
  36: Annual Conference on Neural Information Processing Systems 2023, NeurIPS
  2023, New Orleans, LA, USA, December 10 - 16, 2023}, 2023.
\newblock URL
  \url{http://papers.nips.cc/paper\_files/paper/2023/hash/71b52a5b3fe2e9303433a174b60e160d-Abstract-Conference.html}.

\end{thebibliography}
\bibliographystyle{icml2024}

\newpage
\appendix
\onecolumn
\section{Proofs of Theoretical Results}\label{app:theoretical_proofs}
\subsection{Proof of the Minimax Regret Refinement Theorem (\cref{thrm:minimax_refinement_theorem})}\label{sec:catchall}

\begin{customthm}{\ref{thrm:minimax_refinement_theorem}}
\mmrrefinementtheorem{}
\end{customthm}
\begin{proof}
We first prove (a) inductively. For each $k$, we assume $\pi_{k}$ is minimax regret and must show that $\pi_{k+1}$ is also minimax regret.

Base Case: 
$\pi_1$ is trivially minimax regret by definition.

Inductive Case: Suppose $\pi_k$ is minimax regret.
By \cref{thm:mwci}, $\pi_{k+1}$ must have better worst-case regret than $\pi_k$ for the set of levels outside the support of all previous adversaries. Given this, and since $\pi_{k+1}$ is also constrained to behave identically to $\pi_k$ under the support of these previous adversaries, it cannot perform worse over any previously-sampled level. Thus, $\pi_{k+1}$ cannot decrease worst-case performance compared to $\pi_k$ on the full level set $\Theta$. Hence $\pi_{k+1}$ is minimax regret optimal.

We prove (b) inductively again. 

Base Case: The case of $i = 2$ follows directly from \cref{thm:mwci}.

Inductive Case:
Again, we can invoke \cref{thm:mwci} to show that $\pi_i$ must monotonically improve worst-case regret of $\pi_{i-1}$ over $\Theta \setminus \Theta'_i$.

(c) This follows directly from the definition of the refined game, as $\pi_i$ cannot change behaviour for any level sampled by any previous adversary.

\end{proof}

\subsection{Proof of Bayesian Perfect Policy (\cref{thrm:bayes_perf})}\label{sec:proof:bayes_perf}
\begin{customthm}{\ref{thrm:bayes_perf}}
    \bayesperfecttheorem
\end{customthm}
\begin{proof}
    We denote $\Theta_\hh$ as the information set of $\hh$, consisting of all levels $\theta$ that could have generated $\hh$.

    Let $\pi'$ be a Bayesian level-perfect MMR policy (as per \cref{def:level_perfect}).
    To show that $\pi'$ is a Perfect Bayesian policy, we must show (a) that at every trajectory, corresponding to an information set, $\pi'$ acts optimally with respect to some distribution $\mu(\hh)$ over $\Theta_\hh$. Furthermore, (b) this belief must be updated using Bayes' rule wherever possible, i.e., we do not update the posterior on a 0 probability event.

\textbf{(i): Equlibrium Paths}
    We first consider on-equilibrium paths, i.e., those trajectories reachable by $\pi_1$ and $\Lambda_1$.

    Define
    $$
    W = \min_{\pi \in \Pi}\left\{ \max_{\theta \in \Theta} \{ \perfregret_\theta(\pi) \}\right\}
    $$
    to be the optimal worst-case regret.

    Since $\pi'$ and $\Lambda_1$ are in equilibrium of the minimax regret game, we have that, for each $\theta \in \text{Supp}(P_{\Lambda_1}(\theta))$,  $\perfregret_\theta(\pi') = W$. Therefore, for any probability distribution $P'(\theta)$ such that $\text{Supp}(P'(\theta)) \subseteq \text{Supp}(P_{\Lambda_1}(\theta))$, we have that

    $$
    \mathbb{E}_{\theta \sim P'(\theta)} [\perfregret_\theta(\pi')] = W.
    $$

    Now, let $\hh$ be any trajectory along the equilibrium path. Then we know that $P(\theta|\hh)$ is a probability distribution that has support that is a subset of the support of $\Lambda_1$.

    Assume, for the sake of contradiction, that there exists a policy $\bar\pi$ that obtains a lower expected regret over $P(\theta|\hh)$. 
    If this is true, it must mean that $\bar\pi$ improves upon $\pi'$ in some environments in the support of $P(\theta|\hh)$.
    Since $\text{Supp}(P(\theta | \hh)) \subseteq \text{Supp}(P_{\Lambda_1}(\theta))$, $\bar\pi$ must improve in performance compared to $\pi'$ over the entire distribution $P_{\Lambda_1}(\theta)$. This provides the contradiction, since $\pi'$ and $\Lambda_1$ are at Nash equilibrium---therefore, there does not exist a policy $\bar\pi$ that obtains a better utility while the adversary is fixed.

    \textbf{(ii): Off-Equilibrium Paths}
    Let $\hh$ be a realisable trajectory that is not reached under $\Lambda_1$ and $\pi_1$. Let $i$ be the smallest integer such that $\hh$ is reached under $\Lambda_i$ and $\pi'_i$. Then, using $\Lambda_i$ as the prior and performing Bayesian updating using $\pi'_i$ will result in a well-defined probability distribution over $\Theta_\hh$. 
    If there exists a policy that performs better starting at $\hh$ than $\pi_i'$ over this distribution, then---similarly to case (i)---that contradicts $\pi_i'$ being a Nash equilibrium of the $i$-th refined game. Therefore, by contradiction, $\pi'$ satisfies (a) on off-equilibrium paths.

    \textbf{(iii): Non-Realisable Paths}
    Consider a trajectory that does not occur under any $\Lambda_i$, $\pi'_i$. This must mean that the trajectory can happen in some environment $\theta$, but $\pi'_i$ never acts such that this trajectory occurs. 
    Let $\hh$ denote one of these non-realisable trajectories.
    If $\pi'$ is strictly dominated over the information set $\Theta_\hh$, let $\bar\pi$ be any policy that weakly dominates $\pi'$ and is itself not strictly dominated. This must exist, as some strategies must survive iterated deletion of strictly-dominated strategies~\citep[Proposition 61.2]{osborne1994Course}.
    
    We can modify $\pi'$'s actions in all descendants $\hh'$ of $\hh$ to act according to $\bar\pi$. 
    We call this modified policy $\pi_\text{bayes}$. This policy acts according to $\pi'$ in all realisable trajectories, and therefore obtains the same regret in all environments.
    
    If $\pi'$ is not strictly dominated, we can let $\pi_\text{bayes}$ act according to $\pi'$ on all descendants of $\hh$.

    Then, in either case, we have that $\pi_\text{bayes}$ is not strictly dominated over $\Theta_\hh$, and therefore must be optimal over some distribution $\mu(\Theta_\hh)$ (See Lemma 60.1 of \citet{osborne1994Course}).

    By (i) and (ii), we have that $\pi_\text{bayes}$ satisfies optimality over some belief (which is updated using Bayes' rule) on all realisable paths. By construction, we must have that $\pi_\text{bayes}$ also satisfies this on all non-realisable paths. Therefore, $\pi_\text{bayes}$ is a Perfect Bayesian policy.
    However, since $\pi'$ acts consistently with $\pi_\text{bayes}$ over \textit{all} realisable trajectories (by construction), $\pi'$ acts consistently with some Perfect Bayesian policy on these realisable trajectories.
    This proves the result.

\end{proof}
\section{On the Temporal Inconsistency of Minimax Regret}\label{app:temporal_inconsistency}
In this section we expand upon some simpler, but ultimately flawed alternatives to our BLP formulation.

\textbf{Global Regret} One possible approach is to consider each trajectory $\tau$, and the associated set of levels consistent with it $\Theta_\tau$. We could aim to find the policy that satisfies MMR over all of these subsets. In other words, a policy in the intersection of $\Pi^\star_\text{MMR}$ for all $\Theta_\tau$:
\begin{align}\label{eq:intersect:all}
    \pi \in \Pi_{\mathcal{T}} \dot = \bigcap_{\tau \in \mathcal{T}} \argmin_{\pi \in \Pi}\left\{ \max_{\theta \in \Theta_{\tau}} \{ \perfregret_\theta(\pi) \}\right\}.
\end{align}

However, a simple counterexample suffices to show that this set is not guaranteed to be non-empty in general. Consider a two-step MDP that has a single initial observation but multiple levels $\theta_1$ and $\theta_2$, each having a unique second observation $\tau_1$ or $\tau_2$, respectively. Finally, suppose that $\theta_1$ and $\theta_2$ have different optimal actions in the initial state. Then, $\Theta_{\tau_1} = \{ \theta_1 \}$ and $\Theta_{\tau_2} = \{\theta_2\}$, and the minimax regret policy over $\Theta_{\tau_1}$ and $\Theta_{\tau_2}$ must perform different actions given the shared initial observation. Therefore, the set of minimax regret policies over $\Theta_{\tau_1}$ and $\Theta_{\tau_2}$ must be disjoint. This means that the policy in \cref{eq:intersect:all} is not always guaranteed to exist.

\textbf{Local Regret} Next, if we use a local form of regret, where regret at a trajectory $\tau$ is defined as the performance difference between the optimal agent and the current agent, given that both are initialised to $\tau$.

We propose the following environment as a counterexample to this: A two-step MDP with two levels $\theta_A$ and $\theta_B$. These share an initial state $s_0$, and there are two possible next states, $s_A$ and $s_B$. There is only one allowed action in $s_0$, which stochastically transitions to $s_A$ (99\% probability if the level is $\theta_A$, 1\% probability if the level is $\theta_B$) or $s_B$ (99\% probability in $\theta_B$ and 1\% in $\theta_A$).

Once the agent is in either $s_A$ or $s_B$, it must bet \$100 on whether it is in $\theta_A$ (action $a_1$) or $\theta_B$ (action $a_2$). If it wins the bet it gains \$100, otherwise it loses the \$100.
Any policy can therefore be described by two actions, $i$, $j$, corresponding to the policy's actions in $s_A$ and $s_B$ respectively. \cref{table:lottery:utility} contains the utility for each policy, and \cref{table:lottery:regret} shows the regret for each policy. We note here we consider only deterministic policies, and stochastic policies can be obtained by using a mixed strategy.
 
Now let us consider the minimax regret policy for the entire MDP. 
It is clear that the policy must perform action $a_1$ in $\theta_A$ and $a_2$ in $\theta_B$, with an overall regret of $2$. Therefore, the policy $a_1 a_2$ is MMR.

\begin{table}[H]
    \begin{minipage}{0.5\linewidth}
        \caption{The decision matrix for this MDP.}\label{table:lottery:utility}
        \centering
        \begin{tabular}{ccccc}
            \toprule
            {Utility} & $a_1 a_1$ & $a_1 a_2$ & $a_2 a_1$ & $a_2 a_2$\\
            \midrule
            $\theta_A$ & $+100$ & $+98$ & $-98$ & $-100$   \\
            $\theta_B$ & $-100$  & $+98$ & $-98$  & $+100$  \\
            \bottomrule
        \end{tabular}
    \end{minipage}
    \begin{minipage}{0.5\linewidth}
        \caption{Regret Matrix corresponding to \cref{table:lottery:utility}.}\label{table:lottery:regret}
        \centering
        \begin{tabular}{ccccc}
             \toprule
            {Regret} & $a_1 a_1$ & $a_1 a_2$ & $a_2 a_1$ & $a_2 a_2$\\
            \midrule
            $\theta_A$ & $0$ & $2$ & $198$ & $200$   \\
            $\theta_B$ & $200$  & $2$ & $198$  & $0$  \\
            \bottomrule
        \end{tabular}
    \end{minipage}
\end{table}

Now let us consider the game starting from $s_A$, with utility and regret shown in \cref{table:lottery:utility:sub,table:lottery:regret:sub}. Here we note that the policy can be described solely by one action, that which it takes in $s_A$.

In this case, the game is similar to matching pennies~\citep{osborne1994Course,bonanno2015Game} and we have that the minimax regret policy must perform action $a_1$ and $a_2$, each with 50\% probability. 
By symmetry, the same holds for the game starting at $s_B$. The resulting policy that satisfies minimax regret starting from both $s_A$ and $s_B$ is to act according to $a_1 a_1$ 50\% of the time, and according to $a_2 a_2$ 50\% of the time. If we apply this policy in the original game, we can see that it obtains a worst-case regret of $100$---much higher than the global MMR policy's worst-case regret of $2$.
Therefore, \emph{being minimax regret on both $s_A$ and $s_B$ causes a policy to not be minimax regret over the entire game.}

\begin{table}[H]
    \begin{minipage}{0.5\linewidth}
        \caption{The decision matrix for the MDP starting at $s_A$.}\label{table:lottery:utility:sub}
        \centering
        \begin{tabular}{ccc}
            \toprule
            {Utility} & $a_1$ & $a_2$\\
            \midrule
            $\theta_A$ & $+100$ & $-100$  \\
            $\theta_B$ & $-100$  & $+100$ \\
            \bottomrule
        \end{tabular}
    \end{minipage}
    \begin{minipage}{0.5\linewidth}
        \caption{The regret for the MDP starting at $s_A$.}\label{table:lottery:regret:sub}
        \centering
        \begin{tabular}{ccc}
            \toprule
            {Utility} & $a_1$ & $a_2$\\
            \midrule
            $\theta_A$ & $0$ & $200$  \\
            $\theta_B$ & $200$  & $0$ \\
            \bottomrule
        \end{tabular}
    \end{minipage}
\end{table}

This counterexample shows that if we aim to perform minimax regret using a local regret measure at a particular trajectory $\tau$, then we can invalidate global minimax regret guarantees.

\textbf{Some Intuition} An intuitive explanation for why this happens is that the global policy can hedge its risk by committing 100\% to $\theta_A$ in state $s_A$ and 100\% to $\theta_B$ in $s_B$; in this way, it has minimax regret globally because the environment can transition to either $s_A$ or $s_B$.
The local policy, by contrast, assumes it is \emph{already} in $s_A$ or $s_B$, and therefore has to hedge its risk in each case.

\textbf{Summary} The side-effect of these counterexamples is that we cannot aim to be minimax regret given any trajectory using a local regret measure, because if we change behaviour on any future trajectory, it may cause the agent to lose minimax regret guarantees. This is why our BLP formulation \emph{only} alters behaviour of trajectories that never occur under the previous adversaries and policies. This allows us to circumvent this temporal inconsistency problem.

\section{Additional Background}\label{app:additional_bg}
In this section, we provide additional background. \cref{sec:dui} discusses decisions under ignorance---an important foundational body of work for UED. \cref{sec:practical_ued} discusses practical UED methods and \cref{app:regret_metrics} describes common approximations for regret.
\subsection{Decisions under ignorance}\label{sec:dui}
Decisions under ignorance~\citep{peterson2017Introduction} is the setting where a particular agent has to make a decision, i.e., performing an action, which will have a particular outcome, depending on the true state of the world. The possible world states, actions and outcomes are known, but the probability of the world being in any particular state is unknown.

Generally, this can be formalised by letting $S$ be all possible world states, $A$ being all possible actions that the agent can take, and $U(s, a)$ being the \textit{utility} of the agent performing action $a$ in state $s$.

This in turn can be written as a matrix, where each row corresponds to a particular state $s$ and each column corresponds to a particular action $a$. The entry in the matrix is the utility of performing action $a$ in state $s$. There are several different decision rules of which we briefly discuss three, \textit{minimax}, \textit{leximin} and \textit{minimax regret}.

Concretely, each rule $r$ induces an ordering over actions $\geq_r$. 
The \textbf{minimax} decision rule chooses actions based on maximising the worst case utility, i.e., $a_i \geq_{\text{minimax}} a_j$ if and only if $\min_s U(s, a_i) \geq \min_s U(s, a_j)$, i.e., that the worst case performance of $a_i$ is better than that of $a_j$.

However, if multiple actions have the same worst case utility, then the minimax decision rule is indifferent between these actions. \textbf{leximin} breaks ties by choosing, among those actions that have optimal worst-case utility, the one with the best second worst-case utility, and so on.
Concretely, if we let $\text{min}^n$ as the $n^\text{th}$ worst possible outcome. The leximin ordering is as follows:
\begin{align*}
    a_i \geq_\text{leximin} a_j \text{ iff } \exists n > 0, s.t., min^n(a_i) > min^n(a_j) \text { and } \\min^m(a_i) = min^m(a_j) \text{ for all } m < n.
\end{align*}

\textbf{Minimax regret.} Finally, the \textbf{minimax regret} decision rule chooses actions based on minimising the maximum regret, i.e., $a_i \geq_{\text{minimax regret}} a_j$ if and only if (letting $U^*(s) = \max_a \{ U(s, a) \}$):
\begin{align*}
\max_s \{ U^*(s) - U(s, a_i) \} \leq \max_s \{ U^*(s) - U(s, a_j) \}.
\end{align*}

For instance, if we consider \cref{table:ex_decision_matrix}, $a_1$ and $a_2$ would both be minimax, since their worst case is the best. $a_2$ would be leximin, since its second worst case is better than $a_1$'s second worst case. Both $a_3$ and $a_4$ would be minimax regret, since their worst-case regret is the lowest.

\begin{table}[h]
    \caption{An example decision matrix}\label{table:ex_decision_matrix}
    \centering
    \begin{tabular}{ccccc}
        \toprule
        {} & $a_1$ & $a_2$ & $a_3$  & $a_4$ \\
        \midrule
        $s_1$ & $-100$ & $-100$ & $-101$  & $-101$  \\
        $s_1$ & $-90$  & $-2$   & $-1$ & $1$        \\
        $s_2$ & $-90$  & $0$    & $-1$ & $-2$        \\
        \bottomrule
    \end{tabular}
\end{table}

\begin{table}[h]
    \caption{The regret matrix of \cref{table:ex_decision_matrix}}\label{table:ex_decision_matrix_regret}
    \centering
    \begin{tabular}{ccccc}
        \toprule
        {} & $a_1$ & $a_2$ & $a_3$  & $a_4$ \\
        \midrule
        $s_1$ & $0$   & $0$ & $1$  & $1$  \\
        $s_1$ & $91$  & $3$   & $2$ & $0$        \\
        $s_2$ & $90$  & $0$    & $1$ & $2$        \\
        \bottomrule
    \end{tabular}
\end{table}
\subsubsection{Two-Player zero-sum games}\label{sec:two_player_zero_sum_games}

Here we briefly connect the above setting to two-player zero-sum games. While we have explained the above in the context of a fixed, but unknown world state $s$, the world state could also be chosen by another agent, with its payoff being the negative of the first agent's payoff.

If these players are in a Nash equilibrium, it means that neither player can improve their payoff by unilaterally changing their strategy. If each strategy set is finite, this game has a Nash equilibrium~\citep{nash1950Equilibrium}, and the minimax theorem applies directly~\citep{osborne2004Introduction}, implying 
that the payoff is $\min_{a}\{\max_{s} \{ \text{Payoff(s, a)} \} \}$.

\subsection{Practical UED Methods}\label{sec:practical_ued}
In practice, there are two primary classes of UED methods, generation- and curation-based. 
\textbf{Generation-based methods}, such as PAIRED~\citep{dennis2020Emergent}, train a level-editor adversary using reinforcement learning, and it learns how to generate levels that maximise the agent's regret. As an example, in a maze setting, the adversary performs $N$ sequential actions, and each action places either the agent, the goal, or a wall. It is updated using RL with a reward of $-\text{Regret}_\theta(\pi)$.

Our method is agnostic to the exact UED algorithm used, but our concrete implementation builds on $\text{PLR}^\perp$ \citep{jiang2021Replayguided}.

The other broad class of methods are \textbf{curation-based}~\citep{jiang2021Replayguided,holder2022Evolving,samvelyan2023Maestro}. These methods generate levels using domain randomisation--effectively generating random levels by e.g. sampling maze tiles i.i.d. Then, the agent is evaluated on these levels and its approximate regret is measured. The levels that induce a high regret are then added to a buffer of levels, which the agent is later trained on. In this way, the UED algorithm only has to curate high-regret levels instead of generating them directly, which is quite challenging.

\subsection{Regret Approximations in UED}\label{app:regret_metrics}
While regret has desirable theoretical properties, it is intractable to compute in general. Therefore, UED methods have to instead approximate regret. PAIRED~\citep{dennis2020Emergent} does this by concurrently training two agents, and using the difference in performance between these as a proxy for regret. $\text{PLR}^\perp$~\citep{jiang2021Replayguided} uses two different scoring functions, Positive Value Loss (PVL) and Maximum Monte Carlo (MaxMC). PVL approximates regret as the average of the value loss (the reward obtained minus the predicted reward) over all transitions that have positive value loss in an episode. This prioritises levels on which the agent can still improve. MaxMC approximates the optimal return on any level as the best performance the agent has ever achieved on this level. 
\section{Experimental Details}\label{app:experimental_details}
\subsection{ReMiDi Implementation Details}\label{app:remidi_implementation_details}
We implement ReMiDi on top of $\text{PLR}^\perp$ as our base UED algorithm. 
The procedure is as follows. We first run $\text{PLR}^\perp$ normally for $N = 1000$ iterations. We then initialise a new adversary. 
When new levels are sampled, we perform the agent's action on every one of the levels in the previous buffer(s).
Using the observations we obtain, we determine when the current trajectory is distinguishable from a trajectory from a prior adversary. If the trajectory has complete overlap with any level in the previous buffer, we do not add it to the new one. If it has partial or no overlap, we add it to the new buffer as normal (i.e., if the level has a higher score than any existing one).
We continue to initialise a new adversary after every $N$ steps. For computational reasons, we only have a fixed number of adversaries, and remain at the last one after we have iterated through all of the previous ones.

The above works since each of our environments is deterministic. In stochastic settings, one could either add the random seed to the level $\theta$, in which case the environment becomes deterministic given $\theta$. Alternatively, we can learn a belief model $P(\hh|\Lambda_i)$ concurrently with learning $\Lambda_i$. Then, if $P(\hh|\Lambda_i)$ is lower than some threshold for all previous adversaries, we can consider it as distinguishable.

Finally, we approximate \cref{alg:refinement_alg} in that we do not train separate policies for each adversary. However, we update $\pi$ only on parts of trajectories that are distinguishable from previous adversaries. While we could periodically train the agent on levels from previous buffers to ensure it does not forget, we opt for the simpler option.
\subsection{Environment Description}
\subsubsection{Tabular Setting}
The tabular game consists of 6 environments $\theta_1, \dots, \theta_6$. In the first experiment, each of these consists of a unique trajectory $\hh_i$. In each state, the agent can choose between two actions, $A$ and $B$.
The environment terminates after one step, with rewards given by \cref{tab:app:tabular:one}.
\begin{table}[H]
    \centering
    \caption{The rewards of the first tabular setting in \cref{sec:exp_toy_settings}.}
    \label{tab:app:tabular:one}
    \begin{tabular}{ccll}
         \toprule
         Environment & Trajectory & Reward A & Reward B \\
         \midrule
        $\theta_1$ & $\hh_1$  & $-0.1$ & $+0.1$ \\
        $\theta_2$ & $\hh_2$  & $+0.7$ & $-0.7$ \\
        $\theta_3$ & $\hh_3$  & $-0.7$ & $+0.7$ \\
        $\theta_4$ & $\hh_4$  & $+0.1$ & $-0.1$ \\
        $\theta_5$ & $\hh_5$  & $-1.0$ & $+1.0$ \\
        $\theta_6$ & $\hh_6$  & $+1.0$ & $-1.0$ \\
         \bottomrule
    \end{tabular}
\end{table}

The second experiment used the same rewards, except that $\hh_1 = \hh_2$, $\hh_3 = \hh_4$ and $\hh_5 = \hh_6$.

In both cases, the adversary and agent were implemented using tabular actor-critic~\citep{sutton2018reinforcement}. The procedure is similar to PAIRED~\citep{dennis2020Emergent}, except that we used perfect regret (corresponding to an antagonist that is optimal on each level and does not learn).
Both the agent and adversary are effectively bandit agents, and the adversary was updated using regret as the reward, and the agent's reward was the negative of regret.
When updating the agent, its reward is computed by taking the expectation over the adversary's distribution given the action the agent takes. Likewise, when computing the regret for updating the agent, we compute the agent's expected return over the environment selected by the adversary.
At each iteration, we perform $5$ updates of the adversary and then $5$ updates of the agent.

We use $\gamma=0.95$, entropy coefficient of $0.1$, policy learning rate of $0.01$ and value function learning rate of $1.0$.

\subsubsection{T-maze \& Mazes}
Every time we generate a new random level in this experiment, there is a 50\% chance that it is a T-maze, and it is a normal maze otherwise.
The T-maze's goal is invisible to the agent, and each T-maze looks like the levels in \cref{fig:app:tmazes}. The boundaries are generated randomly. This has the effect that each T-maze in our buffer is unique, even though they play identically. The T-maze terminates once the agent moves with a reward of either $+1$ or $-1$, depending if the goal is reached or not.

\begin{figure}[H]
    \centering
    \begin{subfigure}{0.15\linewidth}
        \includegraphics[width=1\linewidth]{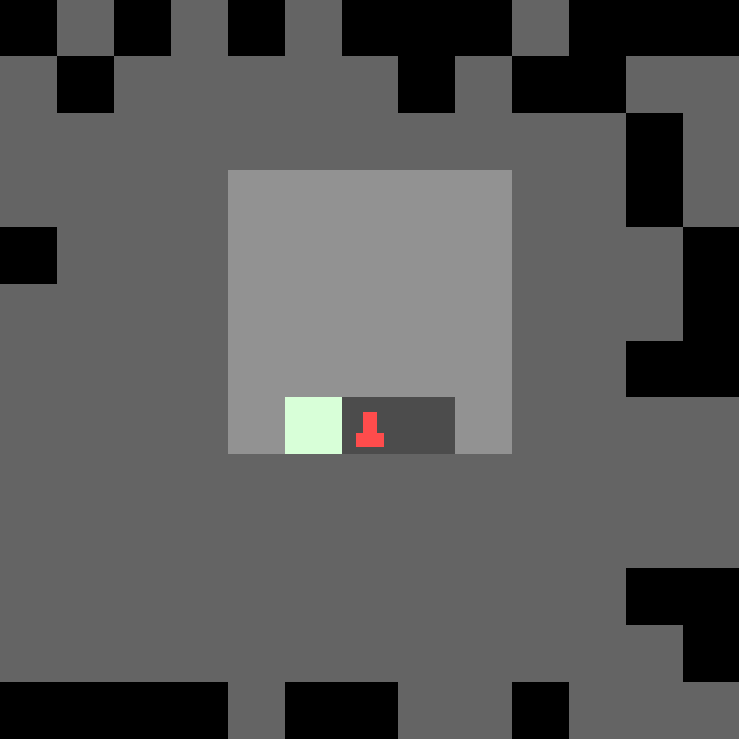}
    \end{subfigure}
    \begin{subfigure}{0.15\linewidth}
        \includegraphics[width=1\linewidth]{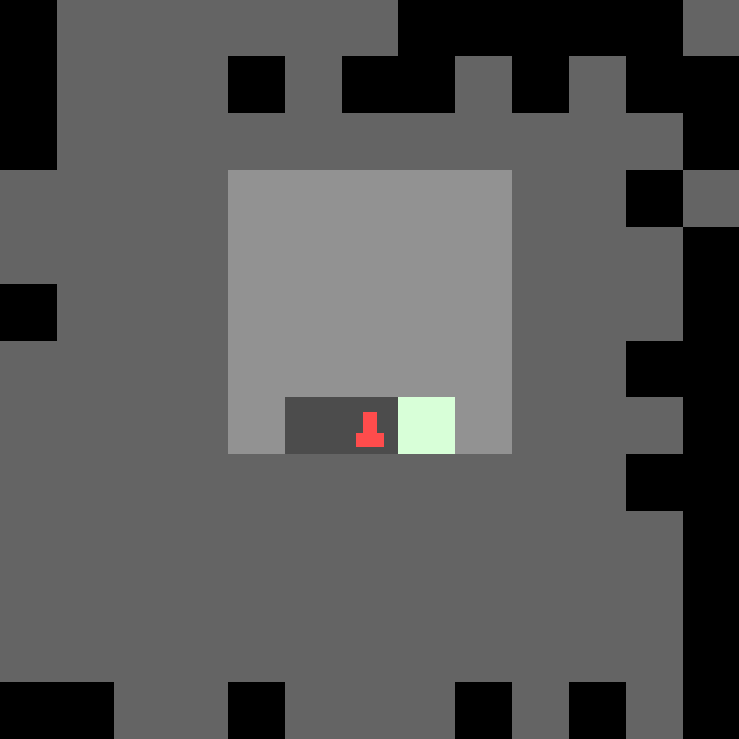}
    \end{subfigure}
    \begin{subfigure}{0.15\linewidth}
        \includegraphics[width=1\linewidth]{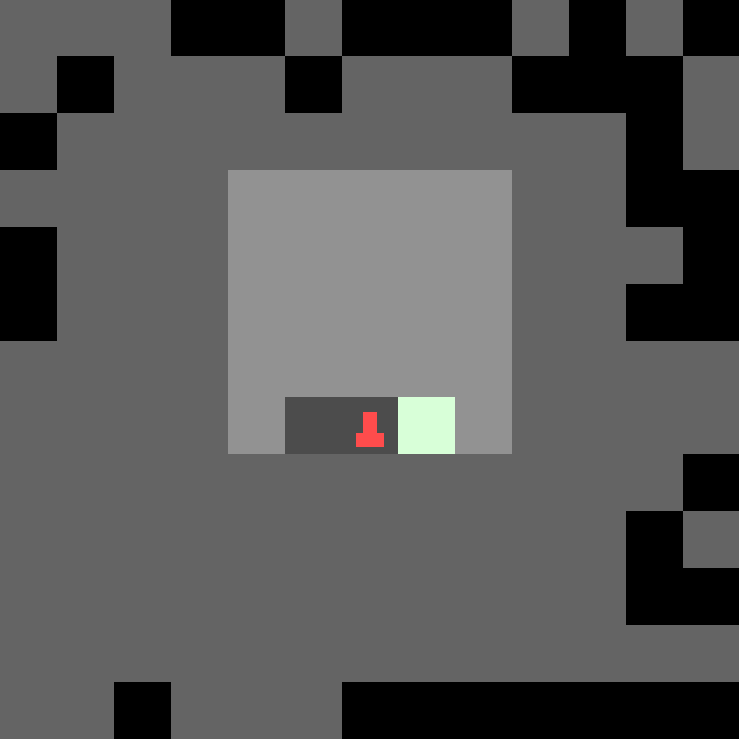}
    \end{subfigure}
    \begin{subfigure}{0.15\linewidth}
        \includegraphics[width=1\linewidth]{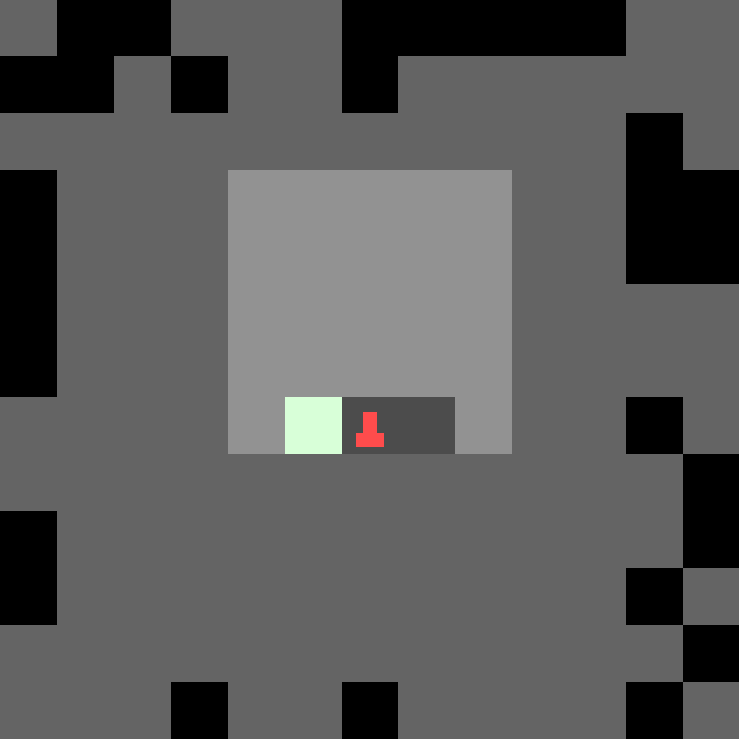}
    \end{subfigure}
    \caption{Several Examples of T-mazes}\label{fig:app:tmazes}
\end{figure}

The mazes are generated using 25 walls sampled IID, similar to prior work~\citep{jiang2021Replayguided}.
We use perfect Monte Carlo regret as the score. This is computed by first finding the shortest path from the agent's start location to the goal, and computing the number of steps that would take. The optimal return is then computed using this. The regret is then computed as the optimal return, minus the average return of the agent on this level.

\paragraph{Agent Architecture}
The agent architecture is similar to that used by prior work~\citep{dennis2020Emergent,jiang2021Replayguided}. In particular, it observes the $5\times 5$ window in front of and including the agent. The agent uses a $3 \times 3$ convolutional layer with 16 channels to process the input image. The agent's direction is embedded into a $5$-dimensional vector. The image embedding is concatenated to the direction embedding and processed using a $256$-feature LSTM~\citep{hochreiter1997Long}.
The output of this is processed by two 32-hidden node fully-connected layers, one for the policy and one for the value estimate.

\paragraph{Evaluation Levels}
We evaluate the agent on a set of held-out standard test mazes used in prior work~\citep{jiang2021Replayguided,holder2022Evolving,jiang2023Minimax}. In particular, we use 
\texttt{SixteenRooms}, \texttt{SixteenRooms2}, \texttt{Labyrinth}, \texttt{LabyrinthFlipped}, \texttt{Labyrinth2}, \texttt{StandardMaze}, \texttt{StandardMaze2}, \texttt{StandardMaze3}, \texttt{SmallCorridor} and \texttt{LargeCorridor}.
\subsubsection{Blindfold}
In the blindfold experiment, levels are generated as normal, and then have a 50\% chance of being ``blindfold'' levels, in which case the agent's observation is filled with zeros.
The agent architecture and evaluation procedure are the same as in the T-maze case.
\subsubsection{Lever Game}
The lever environment is a one-step environment where there are $64$ actions the agent can take. The observation consists of a $65$-dimensional, one-hot-encoded vector. If the first dimension is active, that indicates that the correct answer is hidden. If any of the other $i$ dimensions is activated, then the correct action is $i-1$.
We generate levels by first choosing a correct action and then, with 50\% probability, choosing whether or not the correct answer is visible or not.

The reward for the visible case is $+1$, $-1$ for the correct and the incorrect answers respectively, and $+10$, $-10$ for the invisible case.

The agent's architecture consists of two $256$ hidden node fully-connected layers, and then the same policy and value heads used in the other experiments. For hyperparameters, we use top-k prioritisation with $k=32$, and a buffer size of $64$ for PLR, where ReMiDi has two buffers of size $32$ each.

\subsubsection{Brax}
We use the Brax simulator, and restrict our episodes to only consist of 100 timesteps. The observation space is the standard Brax observations, with a one-hot encoded indication of whether the agent is in a type A or B level. In type B levels, since the joint positions do not matter, they are set to a dummy value of $-1$.
The reward for type B levels is either $+10000$ or $-10000$ depending on if the agent performs the correct (but unobserved) action. Correct here means that the sign of the final action dimension is either positive or negative, depending on the level.

Since we cannot compute the optimal policy in a closed-form manner, we use the heuristic of setting the optimal return per level to be the same as obtained by a policy trained to convergence without disturbances. Since $\pi_\theta$ has access to the level, it can apply the necessary corrections to account for the particular disturbances of level $\theta$.

Finally, during training, we divide the reward obtained by 10 (as it can become quite large), and use a running observation normalisation.
We use the fully connected network architecture inspired by PureJaxRL~\citep{lu2022Discovered}. The network has two fully connected layers, then a recurrent cell, followed by two more fully connected layers for the actor and critic.
\subsection{Hyperparameter Tuning}
\cref{table:ued:hyperparams} contains the hyperparameters we used for our experiments.

We tuned $\text{PLR}^\perp$'s hyperparameters by training on standard mazes and choosing the hyperparameters that obtained the best performance over the evaluation levels. We used these for the minigrid experiments, and made slight changes for the simpler lever game.
We performed a grid search over entropy coefficient of $\{ 0.01, 0.0 \}$, learning rate of $\{ 0.001, 0.0001 \}$, $\lambda$ of $\{ 0.95, 0.98 \}$, temperature of $\{ 0.3, 1.0 \}$ and replay rate of $\{ 0.5, 0.8 \}$. We searched over the same hyperparameters for the Brax experiments, separately for each environment. We use the same parameters for \remidi and minimax return as for PLR.

We used the same hyperparameters for ReMiDi, except that we have an inner buffer size of 256, and 4 outer buffers for the minigrid experiments. The inner buffer size of 256 performed better than a larger buffer of 1000 but is computationally more efficient.
\begin{table}[h!]
    \caption{\small{Hyperparameters}}
    \label{table:ued:hyperparams}
    \begin{center}
    \scalebox{1.0}{
        \begin{tabular}{lrrr}
        \toprule
        \textbf{Parameter}              & Minigrid  & Lever      & Brax \\
        \midrule
        \textbf{PPO}                      &           &           &  \\
        Number of Updates               & 30000     & 2500      & 200 \\
        $\gamma$                        & 0.995     &           & 0.99 \\
        $\lambda_{\text{GAE}}$          & 0.95      &           & 0.95 \\
        PPO number of steps             & 256       &           & 100 \\
        PPO epochs                      & 5         &           & 4 \\
        PPO minibatches per epoch       & 1         &           & 32 \\
        PPO clip range                  & 0.2       &           & \\
        PPO \# parallel environments    & 32        &           & 256 \\
        Adam learning rate              & 0.001     &           & 0.0001 \\
        Anneal LR                       & yes       &           & \\
        Adam $\epsilon$                 & 1e-5      &           & \\
        PPO max gradient norm           & 0.5       &           & \\
        PPO value clipping              & yes       &           & \\
        return normalization            & no        &           & \\
        value loss coefficient          & 0.5       &           & \\
        entropy coefficient     & 0.0       &                   & \\
        \textbf{PLR}                      &           &           & \\
        Replay rate, $p$                & 0.8       & 0.8       & 0.8 \\
        Buffer size, $K$                & 4000      & 64        & 512 \\
        Scoring function                & Perfect   & Perfect   & Perfect \\
        Prioritisation                  & Rank      & Top K     & Rank \\
        Temperature, $\beta$            & 1.0       & -         & 1.0 \\
        $k$                             & -         & 32        & - \\
        Staleness coefficient           & 0.3       & 0.3       & 0.3 \\
        \textbf{ReMiDi}                   &           &           & \\
        Number of Adversaries           & 4         & 2         & 2   \\
        Inner Buffer Size               & 256       & 32        & 512 \\
        Number of Replays per adversary & 1000      & 1000      & 25 \\
        \bottomrule 
        \end{tabular}}
    \end{center}
\end{table}

\end{document}